\documentclass[12pt]{article}
\usepackage[numbers]{natbib}  

\usepackage[margin=2.9cm]{geometry}
\usepackage{enumerate}
\usepackage{xcolor}

\usepackage{graphicx,bm,url,microtype}
\usepackage{paralist}
\usepackage{authblk}
\usepackage{amsfonts,amssymb,amsthm,amsmath, amsthm}

\newtheorem{proposition}{{\sc\bf Proposition}}
\newtheorem{theorem}{{\sc\bf Theorem}}
\newtheorem{lemma}{{\sc\bf Lemma}}
\newtheorem{remark}{{\sc\bf Remark}}

\newtheorem{definition}{{\sc\bf Definition}}
\newtheorem{corollary}{{\sc\bf Corollary}}

\usepackage{amsmath,amsthm,mathrsfs,amsfonts}
\usepackage{url}
\usepackage{amsfonts,todonotes}
\usepackage{amssymb}
\usepackage{bbm}
\usepackage{graphicx}
\usepackage{subfig}
\usepackage{float}

\usepackage{multirow}
\renewcommand{\Pr}[1]	{\mathbb{P}\left( #1 \right)}

	\title{Imbalanced  Classification under Capacity Constraints}
    \author{
Daniel Fraiman \\
\small Departamento de Matemática y Ciencias, Universidad de San Andrés, Buenos Aires, Argentina\\ 
\small CONICET, Argentina. \\
\and
Ricardo Fraiman\\
\small Departamento de Matemática y Ciencias, Universidad de San Andrés, Buenos Aires, Argentina\\ 
\small PEDECIBA, Matemática, Uruguay.
}
	\date{  }
	
\begin{document}

\maketitle

	\begin{abstract}
Detecting observations from a minority class under severe class imbalance is a central challenge in applications such as fraud detection, medical screening, and industrial quality control. In these settings, each positive prediction triggers a costly follow-up action, an MRI scan, a transaction audit, whose execution is subject to real operational constraints. This paper proposes a formal classification framework under capacity constraints: given a user-defined bound limit $b$ on the proportion of observations that can be labeled as belonging to the minority class, the goal is to find the classifier that maximizes sensitivity on that class. We characterize the optimal classifier under this constraint and establish its equivalence with the classical Bayes classifier under a reweighting of the prior probabilities. We also introduce a capacity-adjusted  performance metric $M$ that accounts for the effective detection rate when the capacity constraint is binding. The framework is implemented on top of standard learning methods, k-NN, SVM, random forests, and neural networks, and statistical consistency is established for each. We further show that these methods reduce to post-hoc 
thresholding when no hyperparameters are oriented toward the capacity-constrained objective, and introduce a capacity-aware support vector machine that exploits the constraint during training and achieves the strongest empirical performance. Experiments on the Taiwanese credit card default dataset confirm that capacity-constrained classifiers substantially outperform both classical approaches and  SMOTE under high imbalance regimes. The framework extends naturally to multiclass settings and online environments.

\end{abstract}


\tableofcontents	
	\section{Introduction}

In many classification problems, the goal is not simply to minimize overall error, but to detect as many instances of a rare class as possible under explicit operational constraints. A medical screening program may be able to perform at most fifty follow-up MRI scans per day. A fraud detection system may have a fixed budget of transactions that can be subjected to manual review within the required response window. An industrial quality control pipeline may only be able to pull a limited number of units off the line for detailed inspection. In each of these settings, the practitioner faces a problem that is fundamentally different from standard classification: the constraint is not on accuracy, but on the volume of positive predictions the system is permitted to make.
 
This distinction is not merely a matter of emphasis. It changes the problem. A classifier that achieves high overall accuracy by predominantly predicting the majority class provides essentially no value in these settings. At the same time, a classifier that attempts to detect every minority-class instance will inevitably exceed the available capacity, forcing the practitioner to act on only a random subset of its positive predictions, an outcome that is strictly worse than one that concentrates its detections. The right objective is to maximize detection of the minority class subject to an explicit bound on the rate of positive predictions. We call this the capacity-constrained classification problem.

Despite its practical ubiquity, this formulation has not been studied systematically. The imbalanced classification literature has developed a rich set of tools, reweighting schemes, cost-sensitive learning, and data augmentation procedures such as SMOTE~\cite{smote}, but these methods are designed to improve predictive performance under class imbalance without directly controlling the selection rate. A common post-hoc strategy is to rank observations by classifier score and apply a threshold so that at most a fraction $b\pi_0$ of the population is flagged. While practical, this approach is suboptimal: the score is trained under a different objective, and the resulting ranking is not necessarily aligned with the capacity-constrained problem. We show in Section~\ref{sec:empirical} that explicitly incorporating the capacity constraint into training yields consistent  improvements over post-hoc thresholding.

This paper makes five main contributions. First, we formalize the capacity-constrained classification problem. Given a user-specified parameter $b>0$, we restrict attention to classifiers satisfying
$$
\mathbb{P}(g(X)=0)\leq b\pi_0,
$$
where $Y=0$ denotes the minority class and $\pi_0=\mathbb{P}(Y=0)$. Within this class, the goal is to maximize the detection probability $\mathbb{P}(g(X)=0\mid Y=0)$. The parameter $b$ has a direct operational interpretation: $b=1$ allows the classifier to flag at most a proportion $\pi_0$ of the population, matching the minority-class prevalence, whereas $b>1$ allows a larger flagged fraction, trading specificity for increased coverage of rare events.

Second, we characterize the optimal classifier under this constraint. In the continuous case, the optimal rule selects the region of the feature space where the likelihood ratio $f_0(x)/f_1(x)$ exceeds a threshold $\gamma^\star$, chosen to satisfy the capacity constraint. This result is closely related to the Neyman--Pearson lemma: the capacity constraint plays the role of the size constraint, while the detection probability plays the role of power. We also establish a precise connection with the classical Bayes rule. In particular, the capacity-constrained optimal classifier can be interpreted as a Bayes rule under a modified prior, which provides both theoretical insight and a practical way to adapt standard learning methods through effective class weights.

Third, we introduce a capacity-adjusted performance metric $M$ for comparing classifiers under limited resources. Standard metrics such as sensitivity, specificity, accuracy, and AUC do not directly account for violations of the capacity constraint. In practice, if a classifier flags more observations than can be processed, only a feasible fraction of those flags can be acted upon, reducing the effective detection rate. The metric $M$ captures this effect: it coincides with sensitivity when the capacity constraint is satisfied and scales sensitivity by the feasible fraction when it is violated. Thus, under the resource-limited objective considered here, the relevant quantity is not raw accuracy but capacity-adjusted detection.

Fourth, we establish consistency results for the proposed framework across a broad range of learning procedures, including kernel density estimators, $k$-nearest neighbors, support vector machines, random forests, and neural networks. These results clarify the conditions under which capacity-constrained learning is statistically well behaved. We also show that imposing the capacity constraint does not increase the complexity of the classifier class; in particular,
$$
\mathrm{VCdim}(\mathcal{C}_b)\leq \mathrm{VCdim}(\mathcal{C}),
$$
which is useful for deriving uniform convergence bounds.

Fifth, we introduce capacity-aware learning, a class of methods that incorporate the capacity constraint directly into the training stage rather than enforcing it post-hoc through threshold adjustment. We show that the consistency results of the preceding contribution collapse to post-hoc thresholding when no hyperparameters are available to be oriented toward the capacity-constrained objective, and that genuine gains require exploiting the capacity constraint  during training. As a concrete instance, we develop a  capacity-aware support vector machine and demonstrate empirically that it strictly improves upon both post-hoc thresholding and 
standard capacity-constrained classifiers.

The remainder of the paper is organized as follows. 
Section~\ref{sec:theory} develops the theoretical framework, 
characterizes the optimal classifier, and establishes its connection 
to the Bayes rule. Section~\ref{sec:consistency} establishes 
consistency for each class of learning procedures and evaluates them 
empirically on a real dataset. Section~\ref{sec:beyond} introduces 
capacity-aware learning, characterizes its relationship with post-hoc thresholding, and presents the capacity-aware SVM with an 
empirical illustration. Section~\ref{sec:extensions} discusses  extensions to multiclass settings and online environments. Section 6 concludes.

\subsection*{Relation to existing work}
The imbalanced classification literature is surveyed in~\cite{unbalanced1,rev3,libro1,rev4},
with recent developments in~\cite{review2024,rosenblatt2025} and the references therein. Ensemble methods for imbalanced data are discussed in~\cite{ensembles}, while multiclass extensions are considered in~\cite{multiclass}. Applications to credit card fraud detection can be found in~\cite{credit,fraude}. Cost-sensitive learning methods modify the training loss to penalize minority-class errors more heavily, but do not impose an explicit bound on the selection rate. SMOTE~\cite{smote} addresses imbalance at the data level by generating
synthetic minority-class observations; its limitations in non-convex regions and near class boundaries are well documented. Thresholding methods are discussed in~\cite{threshold} in the context of general imbalanced classification, but without the capacity constraint formulation developed here. Consistency results for plug-in algorithms under confusion-matrix constraints appear in~\cite{consistencia,multiclass}; our contribution is complementary, providing a direct characterization of the optimal rule and closed-form consistency arguments tailored to specific learning procedures rather than general algorithmic frameworks.

\section{Theoretical Framework}
\label{sec:theory}
 
Many important classification problems involve an inherent asymmetry between categories. This asymmetry is not only due to differences in class sizes: crucially, different actions are taken when a new observation is classified into one category or the other. Classifying an observation as belonging to the minority class triggers a costly downstream procedure, a diagnostic scan, a transaction audit, a physical inspection, whose execution is subject to real operational limits. In contrast, classifying an observation as belonging to the majority class is accepted
without further action.
 
We begin by fixing notation and defining the constrained classification problem
(Section~\ref{subsec:setup}). We then characterize the optimal classifier under the capacity constraint (Section~\ref{subsec:optimal}) and establish its precise relationship with the classical Bayes rule (Section~\ref{subsec:bayes}). Finally, we introduce the performance metric appropriate for this setting (Section~\ref{subsec:metric}).
 
\subsection{Setup and notation}
\label{subsec:setup}
 
Let $Y$ denote the class label of a randomly selected observation, where $Y = 0$ corresponds to the minority class and $Y = 1$ to the majority class, and let $X \in \mathbb{R}^d$ denote the associated feature vector. The minority class ($Y = 0$) is the event of interest throughout. Define $\pi_0 = P(Y = 0)$ and $\pi_1 = P(Y = 1) = 1 - \pi_0$, and assume that, conditional on $Y = k$, the random vector 
$X$ has density $f_k$ and distribution function $F_k$, for $k \in \{0, 1\}$. Let
$$
    \mathcal{C} := \{g : \mathbb{R}^d \to \{0,1\}\}
$$
denote the family of measurable classifiers. Sensitivity is defined as
$\mathbb{P}(g(X) = 0 \mid Y = 0)$ and specificity as $\mathbb{P}(g(X) = 1 \mid Y = 1)$.
 
In practice, there is always a maximum capacity for carrying out the additional procedures used
to confirm minority-class membership. We incorporate this limitation directly into the problem
formulation. For a given value $0 < b \leq 1/\pi_0$, define the constrained subclass
\[
    \mathcal{C}_b := \{g \in \mathcal{C} : \mathbb{P}(g(X) = 0) \leq b\pi_0\}.
\]
The quantity $b\pi_0$ is the operational capacity: the maximum proportion of inputs that
may be flagged as belonging to the minority class. The parameter $b$ controls how much the
system is allowed to flag relative to the true minority prevalence. Setting $b = 1$ permits
flagging at most a proportion $\pi_0$ of the population, matching the true minority base rate.
Setting $b > 1$ allows flagging above the base rate, which is useful when the positive class is
hard to identify and achieving adequate sensitivity requires screening more observations than the
expected number of true positives. Note that $b = 0$ reduces to the trivial classifier that never
assigns the minority label.
 
\subsection{The optimal capacity-constrained classifier}
\label{subsec:optimal}
 
Within the constrained class $\mathcal{C}_b$, the goal is to maximize detection of the minority
class. Rather than the classical risk $R(g) = \mathbb{P}(g(X) \neq Y)$, we consider the loss
function
$$
    L(g) = 1 - \mathbb{P}(g(X) = 0,\, Y = 0),
$$
and define the optimal capacity-constrained classifier as
\begin{equation}
    g_b = \operatorname*{arg\,min}_{g \in \mathcal{C}_b} L(g).
    \label{eq:opt}
\end{equation}
Maximizing $\mathbb{P}(g(X) = 0 \mid Y = 0)$ subject to $\mathbb{P}(g(X) = 0) \leq b\pi_0$ is
equivalent to~\eqref{eq:opt}, since $\pi_0$ is fixed.
 
\begin{proposition}
\label{prop:optimal}
Under the above assumptions, the optimal classifier $g_b$ is given by
$$
    g_b(x) =
    \begin{cases}
        0 & \text{if } x \in A_{\gamma^\star}, \\
        1 & \text{if } x \notin A_{\gamma^\star},
    \end{cases}
$$
where $A_\gamma = \{x \in \mathbb{R}^d : \mathrm{d}F_0(x) > \gamma \ \mathrm{d}F_1(x)\}$, and $\gamma^\star$ is the
value of $\gamma$ satisfying
\[
    P_\gamma = \pi_0 F_0(A_\gamma)+ \pi_1 F_1(A_\gamma)   = b\pi_0.
\]

\end{proposition}
\begin{proof}
Let $A$ be the region classified as class 0, so that $g_A(x) = 0$ if $x \in A$ and
$g_A(x) = 1$ otherwise. The problem becomes
$$
    \max_{A} F_0(A)
    \quad \text{subject to} \quad
    \pi_0 F_0(A) + \pi_1 F_1(A) \leq b \pi_0.
$$
 
An optimal set must allocate the available capacity to points that yield the largest contribution
to the objective per unit of marginal probability mass. The relevant ratio is
$$
    \frac{\mathrm{d}F_0(x)}{\pi_0 \ \mathrm{d}F_0(x) + \pi_1 \ \mathrm{d}F_1(x)}
    = \frac{\mathrm{d}F_0(x)/\mathrm{d}F_1(x)}{\pi_0\, \mathrm{d}F_0(x)/\mathrm{d}F_1(x) + \pi_1},
$$
which is strictly increasing in the likelihood ratio $\mathrm{d}F_0(x)/\mathrm{d}F_1(x)$. The optimal region is
therefore $A_{\gamma^\star} = \{x : \mathrm{d}F_0(x)/\mathrm{d}F_1(x) > \gamma^\star\}$, where $\gamma^\star$ is
chosen so that $\mathbb{P}(X \in A_{\gamma^\star}) = b\pi_0$.  If the distribution of
$\mathrm{d}F_0(x)/\mathrm{d}F_1(x)$ has no atom at $\gamma^\star$, equality holds exactly.
\end{proof}

\begin{remark}
The structure of the optimal classifier mirrors that of the Neyman--Pearson lemma. In hypothesis
testing, the most powerful test at significance level $\alpha$ rejects when the likelihood ratio
exceeds a threshold chosen to make the rejection probability exactly $\alpha$. Here, the
capacity constraint $\mathbb{P}(g(X) = 0) \leq b\pi_0$ plays the role of the size constraint,
and sensitivity plays the role of power. The optimal classifier is therefore the analogue of the
most powerful test in the classification setting with a resource budget.
\end{remark}

\begin{remark}
The classical Bayes classifier $g_c = \operatorname*{arg\,min}_{g} R(g)$ assigns label 0 to the
region $\mathcal{B} = \{x : \pi_0 dF_0(x) > \pi_1 dF_1(x)\}$, which corresponds to thresholding
the likelihood ratio at $\pi_1/\pi_0$. The capacity-constrained classifier $g_b$ also
thresholds the likelihood ratio, but at $\gamma^\star$, which is determined by the budget $b$
rather than by the class priors. When the constraint is not binding, that is, when
$\mathbb{P}(g_c(X) = 0) \leq b\pi_0$,  the two classifiers coincide. When the constraint is
binding, $g_b$ expands the decision region for class 0 relative to $g_c$, at the cost of
reduced specificity. Moreover,
$$
    \mathbb{P}(g_b(X) = Y) \geq \pi_1 - b\pi_0(1 - 2\pi_0),
$$
so the overall accuracy of $g_b$ remains bounded below.
\end{remark}
 
\subsection{Relationship with the classical Bayes classifier}
\label{subsec:bayes}
 
The previous remark identifies a structural similarity between $g_b$ and $g_c$: both threshold
the likelihood ratio $f_0(x)/f_1(x)$, differing only in the choice of threshold. This
observation leads to a precise equivalence that clarifies the role of the capacity parameter $b$.
 
\begin{proposition}
\label{prop:bayes_equiv}
For any capacity parameter $b$ and class probabilities $(\pi_0, \pi_1)$, there exists a
modified prior $\pi_0'$ such that the capacity-constrained classifier $g_b$ under
$(f_0, f_1, \pi_0)$ coincides with the classical Bayes classifier under $(f_0, f_1, \pi_0')$.
\end{proposition}
 
\begin{proof}
The classifier $g_b$ thresholds the likelihood ratio at $\gamma^\star$. The Bayes classifier
under prior $(\pi_0', \pi_1')$ thresholds at $\pi_1'/\pi_0'$. Setting $\pi_1'/\pi_0' =
\gamma^\star$ yields the same decision boundary.
\end{proof}

This equivalence has several important consequences. First, it makes the implicit
reweighting transparent. Standard approaches to imbalanced classification, such as
SMOTE or cost-sensitive learning, also amount in effect to modifying the implicit
class weights used during training. However, they do so without a direct criterion
linking the degree of reweighting to the operational constraint. In contrast,
Proposition~\ref{prop:bayes_equiv} makes the reweighting explicit: the modified
prior $\pi_0'$ is uniquely determined by $b$ and $(\pi_0,\pi_1)$ through the
condition $\pi_1'/\pi_0'=\gamma^\star$. Thus, the practitioner specifies $b$ from
observable operational quantities, such as the number of cases the system can process
per unit time, and the corresponding prior adjustment follows automatically.

The relationship between $b$ and $\pi_0'$ is also direct. As $b$ increases, the
capacity constraint becomes less restrictive, the threshold $\gamma^\star$ decreases,
and the modified prior $\pi_0'$ increases relative to $\pi_0$. In the limit
$b\to 1/\pi_0$, the constraint becomes vacuous and the classifier approaches the rule
that always predicts class $0$. Conversely, as $b\to 0$, $\gamma^\star\to\infty$ and
the classifier approaches the rule that always predicts class $1$. The classical Bayes
classifier corresponds to the value of $b$ for which $\gamma^\star=\pi_1/\pi_0$,
that is, the value at which the capacity constraint is exactly satisfied by $g_c$.

This perspective also clarifies the difference between decision-level and data-level corrections. SMOTE does not recover the true minority-class distribution. Rather, it replaces the empirical measure by a mixture of the original sample and a synthetic
distribution supported on line segments connecting neighboring minority observations.
As the oversampling level increases, this synthetic component may progressively
distort the geometry of the minority class, especially near class boundaries, in
non-convex regions, or in the presence of outliers. The capacity-constrained approach, by contrast, operates at the decision level: it modifies the threshold at which a fixed score function is applied, leaving the underlying model unchanged. The correction is therefore transparent, reversible, and grounded directly in the operational constraint.

Finally, it is worth noting that the proposed classifier with $b=1$ coincides with the Bayes-optimal classifier not only when the classes are clearly separable, but also whenever the capacity constraint is already satisfied by $g_c$. Figure~\ref{comparisonT} illustrates this point for the one-dimensional example described below.

\medskip
\subsubsection{A one-dimensional illustration.}
	\begin{figure}[h]
		\centering
		\includegraphics[width=0.95\textwidth]{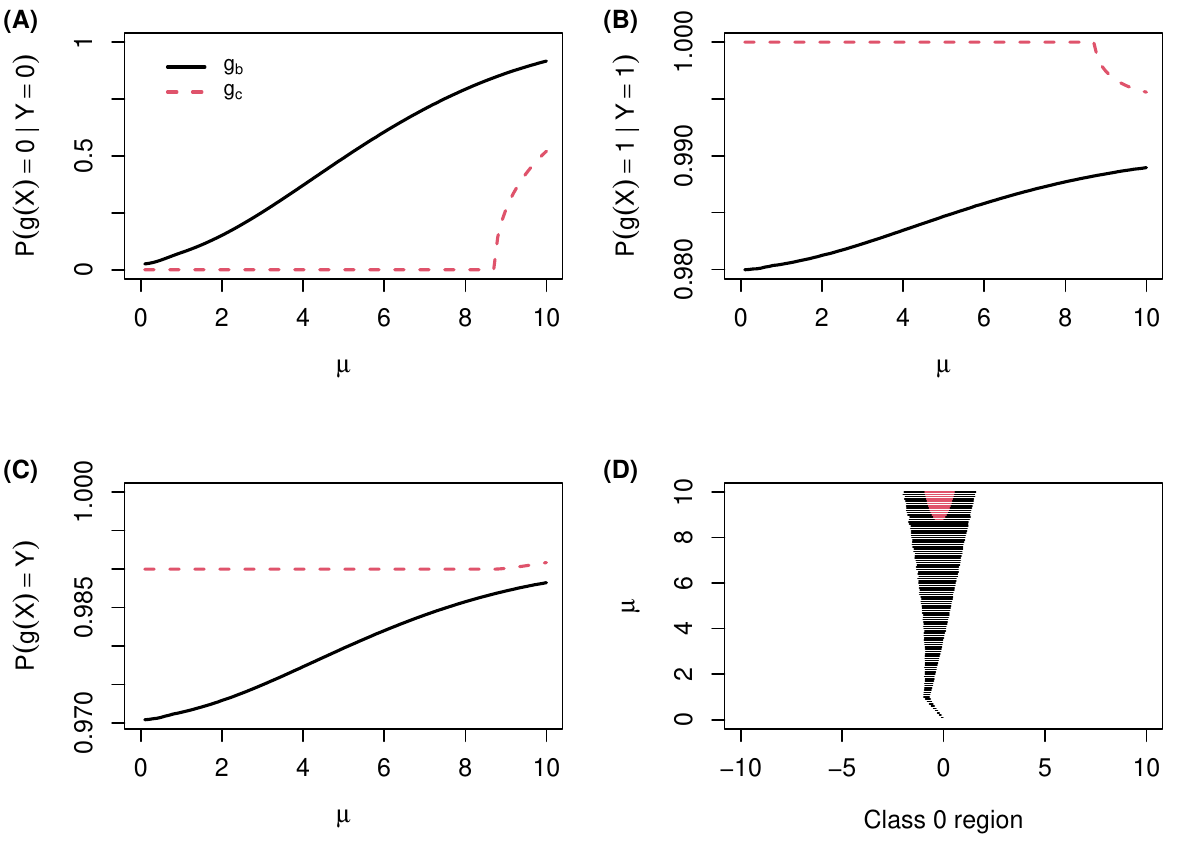}
		\caption{Comparison between the classical optimal classification rule (which maximizes $\Pr{g(X)=Y}$) and our proposed procedure for $b=2$, and $\pi_0=0.01$. Results for the classical optimal rule are shown in red, and results for our method in black. Panels: (A) $\Pr{g(X)=0 \mid Y=0}$, (B) $\Pr{g(X)=1 \mid Y=1}$, and (C) $\Pr{g(X)=Y}$, all as functions of $\mu$. (D) Comparison of the regions of $x$ where $g(x)=0$ under the proposed (black) and classical (red) classifiers; the classical region partially overlaps and covers the proposed region.
		}
		\label{comparisonT}
	\end{figure}

To build intuition, consider $d = 1$ with $F_0(x) = \Phi(x)$, the standard Normal distribution, and $f_1(x) = \{\pi(1 + (x-\mu)^2)\}^{-1}$, a Cauchy distribution shifted by $\mu > 0$. Under severe class imbalance (e.g., $\pi_0 = 0.01$), the Bayes classifier assigns all observations to
class 1 for $\mu < \mu_c \approx 8.721$, since the likelihood ratio never exceeds $\pi_1/\pi_0$ over a set of positive probability. The capacity-constrained classifier $g_b$, by contrast, always identifies a non-empty region $A_{\gamma^\star}$ as class 0, yielding strictly positive sensitivity for all values of $\mu$. Figure~\ref{comparisonT} compares the two classifiers for $b = 2$ and $\pi_0 = 0.01$ across a range of $\mu$ values. Panel~(A) shows that $g_b$ achieves substantially higher sensitivity than $g_c$ throughout, while Panel~(C) shows that overall accuracy under $g_b$ is a strictly increasing function of $\mu$, in contrast to the non-monotone behavior of $g_c$. Panel~(D) illustrates how the class-0 region $A_{\gamma^\star}$ expands as $\mu$ increases and the two classes become more separable.
 
\subsection{A performance metric for capacity-constrained classification}
\label{subsec:metric}
 
In unconstrained classification, overall accuracy $\mathbb{P}(g(X) = Y)$ is a natural summary of performance, though many alternatives exist. Under a capacity constraint, the objective is unambiguous: detect as many minority-class instances as possible subject to
$\mathbb{P}(g(X) = 0) \leq b\pi_0$. This suggests that sensitivity
$\mathbb{P}(g(X) = 0 \mid Y = 0)$ is the primary performance criterion, and $g_b$ is optimal under it by construction.
 
However, a subtlety arises when comparing classifiers that do not all satisfy the constraint. If a classifier $g$ satisfies $\mathbb{P}(g(X) = 0) > b\pi_0$, then not all of its positive
predictions can be acted upon. Under the natural assumption that the subset of flagged observations that can be processed is selected at random, only a fraction
$$
    w := \frac{b\pi_0}{\mathbb{P}(g(X) = 0)}
$$
of the flags are processed, and the effective detection rate is reduced proportionally. This
motivates the following metric.
 
\begin{definition}
\label{def:M}
The capacity-adjusted detection rate of a classifier $g$ is
\[
    M(g) := \mathbb{P}(g(X) = 0 \mid Y = 0) \times
    \begin{cases}
        1 & \text{if } \mathbb{P}(g(X) = 0) \leq b\pi_0, \\[4pt]
        \dfrac{b\pi_0}{\mathbb{P}(g(X) = 0)} & \text{if } \mathbb{P}(g(X) = 0) > b\pi_0.
    \end{cases}
\]
\end{definition}
 
By construction, $g_b$ maximizes $M$ over all classifiers in $\mathcal{C}$. The metric $M$
differs from sensitivity in that it penalizes classifiers that violate the constraint, and it
differs from accuracy in that it focuses entirely on minority-class detection. Standard metrics
such as AUC or F1 score do not respect the capacity constraint and are therefore not appropriate for comparing classifiers in this setting.
 
In practice, $M$ is estimated by its empirical counterpart $\hat{M}$, obtained by replacing population probabilities with their empirical estimates on a held-out test sample. All empirical
comparisons in Section~\ref{sec:empirical} are reported in terms of $\hat{M}$.
 
\section{Statistical Consistency and Empirical Evaluation}
\label{sec:consistency}

We now turn to the question of whether the empirical capacity-constrained classifiers introduced
in Section~\ref{sec:theory} converge to their population-level counterparts as the sample size grows. We establish consistency for five families of learning procedures: kernel density estimators, $k$-nearest neighbors, support vector machines, random forests, and neural networks.

The section is organized as follows. We begin with a general result
(Section~\ref{subsec:general}) showing that imposing the capacity constraint reduces the effective complexity of the classifier class, so that any consistent procedure over $\mathcal{C}$ remains consistent over $\mathcal{C}_b$. We then treat each learning procedure in turn (Sections~\ref{subsec:kernel}--\ref{subsec:nn}), providing full proofs in each case. The general
result furnishes the main argument for neural networks and ERM-based procedures, while the remaining proofs exploit the specific structure of each method.

Throughout, we assume an i.i.d.\ sample $\mathcal{D}_n = \{(X_1, Y_1), \ldots, (X_n, Y_n)\}$ drawn from the same distribution as $(X, Y)$. The sample is split into three disjoint parts: a training set $A_1$ of size $n_1$, a calibration set $A_2$ of size $n_2$ used to estimate the
threshold $\gamma^\star$, and a test set $B$ of size $n_3$, with $n = n_1 + n_2 + n_3$. The empirical loss on the training set is
$$
    \hat{L}_n(g) = 1 - \frac{1}{|S_{\mathrm{train}}|}
    \sum_{j \in S_{\mathrm{train}}} \mathbf{1}_{\{g(X_j)=0,\, Y_j=0\}},
$$
and the empirical capacity-constrained classifier is
$$
    \hat{g}_b = \operatorname*{arg\,min}_{g \in \mathcal{C}_b} \hat{L}_n(g).
$$

\subsection{A general complexity reduction result}
\label{subsec:general}

The key observation underlying all consistency results in this section is that restricting to $\mathcal{C}_b$ does not increase the complexity of the classifier class. This is formalized through the VC dimension and shatter coefficients.

Recall that the shatter coefficient of a class $\mathcal{C}$ is
$$
    S(\mathcal{C}, n) = \max_{x_1, \ldots, x_n \in \mathbb{R}^d}
    \bigl|\{(g(x_1), \ldots, g(x_n)) : g \in \mathcal{C}\}\bigr|,
$$
and the VC dimension is $\mathrm{VCdim}(\mathcal{C}) = \max\{n : S(\mathcal{C}, n) = 2^n\}$.

\begin{theorem}
\label{thm:vc}
$S(\mathcal{C}_b, n) \leq S(\mathcal{C}, n)$ for all $n$. In particular,
$\mathrm{VCdim}(\mathcal{C}_b) \leq \mathrm{VCdim}(\mathcal{C})$.
\end{theorem}

\begin{proof}
Since $\mathcal{C}_b \subseteq \mathcal{C}$, the set of labelings of any finite sample achievable by classifiers in $\mathcal{C}_b$ is a subset of those achievable by classifiers in $\mathcal{C}$. Therefore $S(\mathcal{C}_b, n) \leq S(\mathcal{C}, n)$ for all $n$, and the inequality on VC dimensions follows immediately from the definition.

To see the strict mechanism, observe that the family of error sets associated with $\mathcal{C}_b$,
$$
    \{(x,y) \in \mathbb{R}^d \times \{0,1\} : g(x) \neq 0,\, y = 0\},
    \quad g \in \mathcal{C}_b,
$$
is a strict subset of the family of error sets associated with $\mathcal{C}$,
$$
    \{(x,y) \in \mathbb{R}^d \times \{0,1\} : g(x) \neq y\},
    \quad g \in \mathcal{C}.
$$
The shatter coefficient of the first family is therefore no larger than that of the second.
\end{proof}

\begin{corollary}
\label{cor:erm}
If $\mathrm{VCdim}(\mathcal{C}) < \infty$, then
\[
    \sup_{g \in \mathcal{C}_b} \bigl|\hat{L}_n(g) - L(g)\bigr|
    \leq \sup_{g \in \mathcal{C}} \bigl|\hat{L}_n(g) - L(g)\bigr|
    \xrightarrow{a.s.} 0.
\]
In particular, if $g_b$ is the minimizer of $L$ over $\mathcal{C}_b$ and $\hat{g}_b$ is the empirical minimizer, then $L(\hat{g}_b) \to L(g_b)$ almost surely.
\end{corollary}

\begin{proof}
The first inequality follows directly from Theorem~\ref{thm:vc}. The almost sure convergence of the supremum follows from the VC inequality (Vapnik \& Chervonenkis, 1971) together with the
Borel--Cantelli lemma, since $\mathrm{VCdim}(\mathcal{C}_b) \leq \mathrm{VCdim}(\mathcal{C}) < \infty$. The convergence of $L(\hat{g}_b)$ to $L(g_b)$ then follows from the standard
decomposition
$$
    L(\hat{g}_b) - L(g_b)
    \leq L(\hat{g}_b) - \hat{L}_n(\hat{g}_b)
       + \hat{L}_n(\hat{g}_b) - \hat{L}_n(g_b)
       + \hat{L}_n(g_b) - L(g_b)
    \leq 2 \sup_{g \in \mathcal{C}_b} |\hat{L}_n(g) - L(g)|,
$$
where the second inequality uses the fact that $\hat{g}_b$ minimizes $\hat{L}_n$ over
$\mathcal{C}_b$, so $\hat{L}_n(\hat{g}_b) \leq \hat{L}_n(g_b)$.
\end{proof}

The following auxiliary results are used repeatedly in the proofs below. The first shows that consistency of the estimated threshold $\hat\gamma^\star$ follows from uniform convergence of the score functions.

\begin{lemma}
\label{lem:sets}
Let $h_n(x) = f_n(x) - \gamma_n g_n(x)$ and $h(x) = f(x) - \gamma g(x)$, where $f_n, g_n, f,
g$ are densities on $\mathbb{R}^d$. Define
$$
    E_n = \{x : f_n(x) > \gamma_n g_n(x)\},
    \qquad
    E = \{x : f(x) > \gamma g(x)\}.
$$
If $f_n \to f$ and $g_n \to g$ in measure, $\gamma_n \to \gamma$ in probability, and
$\mu(\{x : f(x) = \gamma g(x)\}) = 0$ for some probability measure $\mu$, then
$\mu(E_n \,\Delta\, E) \to 0$.
\end{lemma}

\begin{proof}
Since $f_n \to f$ and $g_n \to g$ in measure and $\gamma_n \to \gamma$ in probability, $h_n =
f_n - \gamma_n g_n$ converges in probability to $h = f - \gamma g$. By the hypothesis that
$\mu(\{x : h(x) = 0\}) = 0$, the function $h$ is nonzero $\mu$-almost everywhere. Therefore
$\mathbf{1}_{E_n}(x) = \mathbf{1}_{\{h_n(x) > 0\}}$ converges $\mu$-almost surely to
$\mathbf{1}_E(x) = \mathbf{1}_{\{h(x) > 0\}}$. By dominated convergence,
$$
    \mu(E_n \,\Delta\, E)
    = \int |\mathbf{1}_{E_n} - \mathbf{1}_E|\,d\mu \to 0.
$$
\end{proof}

\begin{proposition}
\label{thm:threshold}
Let $\Gamma = [\underline\gamma, \bar\gamma] \subset \mathbb{R}$ be a compact interval. Let
$h : \Gamma \to \mathbb{R}$ be continuous and let $\hat{h} : \Gamma \to \mathbb{R}$ be a
sequence of random functions satisfying
\[
    \sup_{\gamma \in \Gamma} |\hat{h}(\gamma) - h(\gamma)|
    \xrightarrow{\,P\,} 0.
\]
Define feasible sets $F = \{\gamma \in \Gamma : h(\gamma) \geq 0\}$ and
$\hat{F} = \{\gamma \in \Gamma : \hat{h}(\gamma) \geq 0\}$.
Assume:
\begin{itemize}
    \item \bf A1 \rm $F \neq \emptyset$, and $\hat{F} \neq \emptyset$ with probability tending to one.
    \item \bf A2 \rm There exists a unique $\gamma_0 \in \Gamma$ such that $\gamma_0 = \min F$, and for
    every $\varepsilon > 0$,
    \[
        \sup_{\gamma \leq \gamma_0 - \varepsilon} h(\gamma) < 0,
        \qquad
        \inf_{\gamma \geq \gamma_0 + \varepsilon} h(\gamma) > 0.
    \]
\end{itemize}
Define $\hat\gamma = \min \hat{F}$. Then $\hat\gamma \xrightarrow{P} \gamma_0$.
\end{proposition}

\begin{proof}
Fix $\varepsilon > 0$. By assumption (A2), define
\[
    a_\varepsilon = -\sup_{\gamma \leq \gamma_0 - \varepsilon} h(\gamma) > 0,
    \qquad
    b_\varepsilon = \inf_{\gamma \geq \gamma_0 + \varepsilon} h(\gamma) > 0,
    \qquad
    \delta_\varepsilon = \tfrac{1}{2}\min\{a_\varepsilon, b_\varepsilon\}.
\]
Consider the event $\mathcal{E}_n(\varepsilon) = \{\sup_{\gamma \in \Gamma} |\hat{h}(\gamma) -
h(\gamma)| \leq \delta_\varepsilon\}$, which satisfies $\mathbb{P}(\mathcal{E}_n(\varepsilon))
\to 1$ by hypothesis.

On $\mathcal{E}_n(\varepsilon)$, for any $\gamma \leq \gamma_0 - \varepsilon$:
\[
    \hat{h}(\gamma) \leq h(\gamma) + \delta_\varepsilon \leq -a_\varepsilon + \delta_\varepsilon
    \leq -\delta_\varepsilon < 0,
\]
so $\gamma \notin \hat{F}$, which gives $\hat\gamma \geq \gamma_0 - \varepsilon$.

For any $\gamma \geq \gamma_0 + \varepsilon$:
\[
    \hat{h}(\gamma) \geq h(\gamma) - \delta_\varepsilon \geq b_\varepsilon - \delta_\varepsilon
    \geq \delta_\varepsilon > 0,
\]
so $\gamma \in \hat{F}$, which gives $\hat\gamma \leq \gamma_0 + \varepsilon$.

Hence on $\mathcal{E}_n(\varepsilon)$, $|\hat\gamma - \gamma_0| \leq \varepsilon$, and
therefore $\mathbb{P}(|\hat\gamma - \gamma_0| > \varepsilon) \leq 1 -
\mathbb{P}(\mathcal{E}_n(\varepsilon)) \to 0$.
\end{proof}

\subsection{Kernel density-based classifiers}
\label{subsec:kernel}

We assume that sufficiently many observations are available in each class. Although the classes
may be severely imbalanced (e.g., $\pi_0 = 0.01$ or even $0.001$), the absolute sample sizes
remain large enough for reliable density estimation~--- a realistic assumption in applications
where data accumulate at high volume.

From the training sample $A_1$, we construct kernel density estimators for each class:
\[
    \hat{f}_k(x) = \frac{1}{n_{1,k}\, h^d}
    \sum_{i=1}^{n_1} \mathbf{1}_{\{Y_i = k\}} K\!\left(\frac{x - X_i}{h}\right),
    \quad k = 0, 1,
\]
where $n_{1,k} = \sum_{i=1}^{n_1} \mathbf{1}_{\{Y_i = k\}}$ and $K$ is a regular kernel. These
estimators satisfy the conditions of Theorem~10.1 in \citet{DGL} for $L_1$ consistency:
$\int |\hat{f}_k(x) - f_k(x)|\,dx \to 0$ almost surely as $n_{1,k} \to \infty$, provided the
bandwidth $h = h_n$ satisfies $h_n \to 0$ and $n_{1,k} h_n^d \to \infty$.

From the calibration sample $A_2$, we estimate the threshold $\gamma^\star$ as follows. For each
$\gamma$, define $H_\gamma = \{x \in A_2 : \hat{f}_0(x) > \gamma \hat{f}_1(x)\}$ and
\[
    \hat{P}_\gamma = \frac{1}{n_2} \sum_{j=n_1+1}^{n_1+n_2} \mathbf{1}_{H_\gamma}(X_j).
\]
The estimated threshold is $\hat\gamma^\star = \arg\min_\gamma \{\hat{P}_\gamma - b\pi_0 \geq
0\}$, and the classifier is
\[
    \hat{g}_b(x) =
    \begin{cases}
        0 & \text{if } \hat{f}_0(x) > \hat\gamma^\star \hat{f}_1(x), \\
        1 & \text{otherwise.}
    \end{cases}
\]

\begin{theorem}
\label{thm:kernel_consistency}
Under the bandwidth conditions above, and provided $\mu(\{x : f_0(x) = \gamma^\star f_1(x)\}) =
0$ where $\mu$ denotes the marginal distribution of $X$, the kernel-based capacity-constrained
classifier is consistent:
\[
    \mathbb{P}(\hat{g}_b(X) \neq g_b(X)) \to 0 \quad \text{as } n \to \infty.
\]
\end{theorem}

\begin{proof}
We show that both class-conditional error probabilities converge to zero.

\textit{Step 1: convergence of $\hat\gamma^\star$ to $\gamma^\star$.}
Define $h(\gamma) = P_\gamma - b\pi_0$, where $P_\gamma = \mathbb{P}(f_0(X)/f_1(X) > \gamma)$.
By the $L_1$ consistency of $\hat{f}_0$ and $\hat{f}_1$, the empirical version $\hat{P}_\gamma$
converges uniformly in $\gamma$ over any compact interval $\Gamma$ to $P_\gamma$. The function
$h(\gamma)$ is continuous and strictly decreasing in $\gamma$ under the assumption that the
distribution of the likelihood ratio $f_0(X)/f_1(X)$ has no atoms; this guarantees the
uniqueness condition (A2) of Theorem~\ref{thm:threshold}. Therefore $\hat\gamma^\star
\xrightarrow{P} \gamma^\star$ by Theorem~\ref{thm:threshold}.

\textit{Step 2: convergence of the decision regions.}
Define $A_1 = \{(x,y) : f_0(x) > \gamma^\star f_1(x),\, y = 1\}$ and $\hat{A}_{n,1} = \{(x,y)
: \hat{f}_0(x) > \hat\gamma^\star \hat{f}_1(x),\, y = 1\}$, and similarly $A_0$, $\hat{A}_{n,0}$
for class 0. By Lemma~\ref{lem:sets}, applied with $f_n = \hat{f}_0$, $g_n = \hat{f}_1$, and
$\gamma_n = \hat\gamma^\star$, together with the $L_1$ consistency of the kernel estimators and
Step~1, we obtain $\mu(A_k \,\Delta\, \hat{A}_{n,k}) \to 0$ for $k = 0, 1$.

\textit{Step 3: conclusion.}
The misclassification probabilities are
\begin{align*}
    \mathbb{P}(\hat{g}_b(X) \neq g_b(X) \mid Y = 1) &= \int_{C_{n,1}} f_1(x)\,dx, \\
    \mathbb{P}(\hat{g}_b(X) \neq g_b(X) \mid Y = 0) &= \int_{C_{n,0}} f_0(x)\,dx,
\end{align*}
where $C_{n,k} = A_k \,\Delta\, \hat{A}_{n,k}$. Since $\mu(C_{n,k}) \to 0$ and $f_k$ is
integrable, both integrals converge to zero by absolute continuity of the Lebesgue integral.
\end{proof}
To illustrate the behavior of the kernel-based capacity-constrained classifier, we present
a simulation study in $\mathbb{R}^2$. Class 0 is uniformly distributed within a small ellipse,
and class 1 is uniformly distributed within a larger ellipse containing the smaller one
(Figure~\ref{compa2D}, Panel~A). The number of class-0 observations is fixed at $n_0 = 100$,
and the minority class proportion $\pi_0$ varies across experiments.

Panel~B shows $\hat{\mathbb{P}}(\hat{g}(X) = 0 \mid Y = 0)$ as a function of $\pi_0$ for
$\hat{g}_{b=1}$, $\hat{g}_{b=2}$, and the classical rule $\hat{g}_c$. The proposed classifiers
consistently achieve higher detection rates than $\hat{g}_c$, particularly for small and
moderate $\pi_0$. Panel~C shows the corresponding majority-class accuracy
$\hat{\mathbb{P}}(\hat{g}(X) = 1 \mid Y = 1)$, illustrating the fundamental trade-off: improved
minority detection comes at the cost of reduced specificity. Panel~D reports the
capacity-adjusted metric $\hat{M}$. Once capacity is accounted for, the advantage of the
classical Bayes classifier over $\hat{g}_{b=1}$ that appears in Panel~B disappears entirely,
while $\hat{g}_{b=2}$ remains dominant throughout. This confirms that $\hat{M}$ is the
appropriate criterion for comparison in this setting, and that the Bayes classifier is not
optimal under capacity constraints.

\begin{figure}
    \centering
    \includegraphics[width=1.\textwidth]{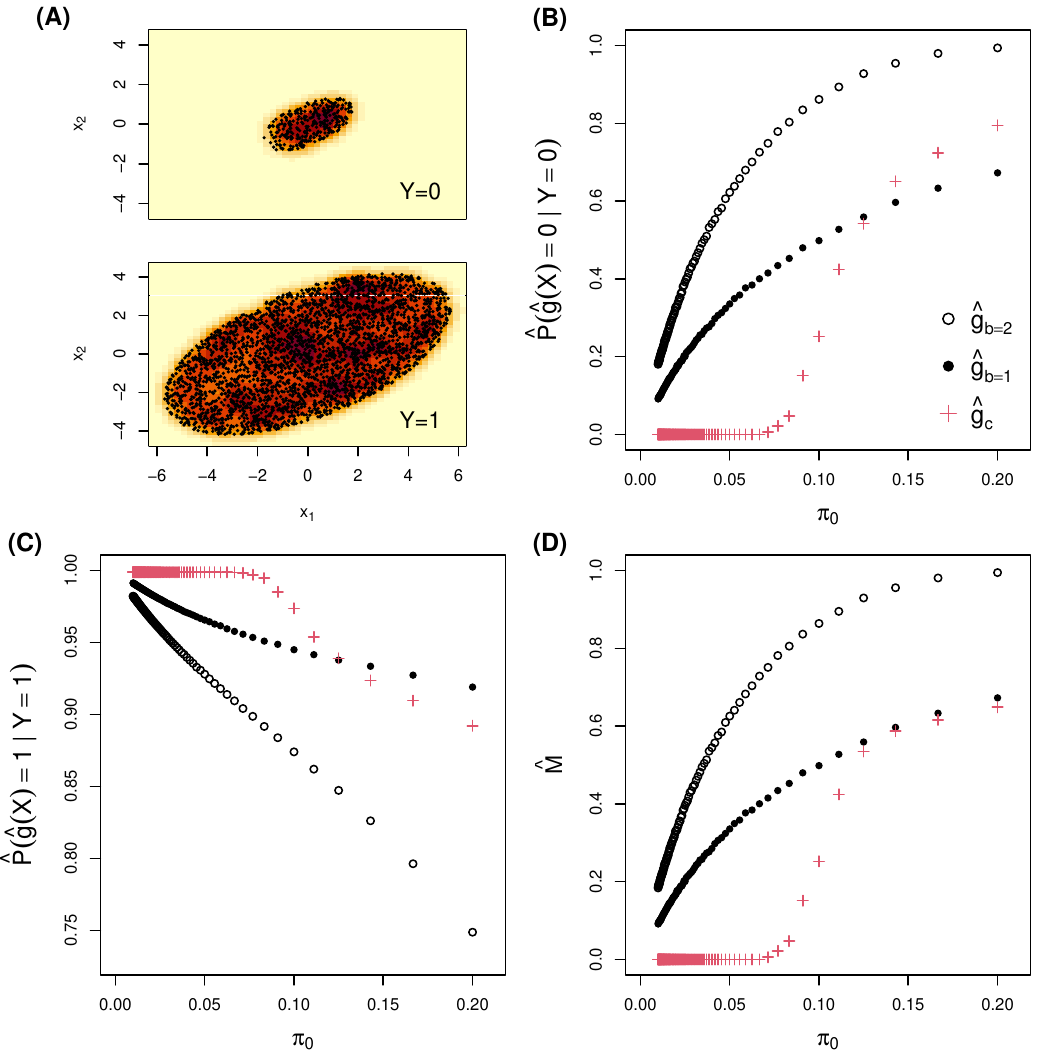}
    \caption{Comparison between the classical optimal classification rule and our proposed
    procedure for $b = 1, 2$, and different values of $\pi_0$. Panels: (A) An example of the
    data corresponding to category $Y=0$ and category $Y=1$; (B)
    $\hat{\mathbb{P}}(\hat{g}(X) = 0 \mid Y = 0)$; (C)
    $\hat{\mathbb{P}}(\hat{g}(X) = 1 \mid Y = 1)$; (D) $\hat{M}$, all as functions of $\pi_0$.
    Results for the classical rule are shown in red; results for the proposed classifiers in
    black.}
    \label{compa2D}
\end{figure}
\subsection{$k$-nearest neighbors}
\label{subsec:knn}

Let $d(\cdot, \cdot)$ be a metric on $\mathbb{R}^p$ and let $\mathcal{N}_k(x)$ denote the set
of the $k$ nearest neighbors of $x$ in the training sample. The capacity-constrained $k$-NN
classifier is obtained by modifying the class weights. For $0 < a_0, a_1 < 1$ with $a_0 + a_1
= 1$, define
$$
    M_j(x) = \frac{1}{k} \sum_{i=1}^n (a_j \mathbf{1}_{\{Y_i = j\}})
    \mathbf{1}_{\{X_i \in \mathcal{N}_k(x)\}}, \quad j = 0, 1,
$$
and classify $x$ as class 0 if $M_0(x) > M_1(x)$, 

		\begin{equation}\label{clasi}
			g_{a_0}(x) =
			\begin{cases}
				0 & \text{if }  M_0(x)> M_1(x),\\
				1 & \text{if }  \text{otherwise}.
			\end{cases}
		\end{equation}
 The weight $a_0$ controls the effective
upweighting of the minority class. Given a calibrating sample of size $r$, define
$$
    \hat{N}_0 = \sum_{i=1}^r \mathbf{1}_{\{g_{a_0}(X_i) = 0\}},
    \qquad
    \hat{N}_1 = r - \hat{N}_0.
$$
The parameter $a_0$ is chosen to maximize $\hat{N}_0$ subject to $\hat{N}_0/r \leq b\pi_0$,
that is,
\[
    \hat{a}_0 = \arg\max \left\{ a_0 \in (0,1) :
    \frac{1}{r}\sum_{i=1}^r \hat{N}_0 \leq b\pi_0 \right\}.
\]

Royall~\cite{royal1966} introduced the concept of the weighted $k$-NN estimator for nonparametric regression and classification, given by

\begin{equation} \label{weighted}
\hat{\eta}_{n}(x)=\sum _{i=1}^{n}w_{ni}(x)Y^{(i)},	
\end{equation}
where $w_{ni}(x)$ represents a sequence of non-negative weights assigned to the $(i)$-th nearest neighbor of $x$. He showed that weighted $k$-NN works exceptionally well under ideal, smooth conditions on the regression function, and obtain rates of convergence. 
Stone's seminal work~\cite{stone} generalized Royall's formulation and  proved that nearest-neighbor type estimators given by (\ref{weighted}) are robust and theoretically guaranteed to converge even when the true underlying data structure is highly irregular or unknown (Universal Consistency). See for instance Theorem 6.3 in \cite{DGL}. 

\begin{theorem}
\label{thm:knn_consistency}
Assume $k = k_n \to \infty$ and $k_n/n \to 0$ as $n \to \infty$. Then the capacity-constrained
$k$-NN classifier is consistent: $L(\hat{g}_b) \to L(g_b)$ in probability.
\end{theorem}

\begin{proof}
The standard $k$-NN classifier with fixed weights is consistent under the conditions $k \to
\infty$ and $k/n \to 0$ by the results of~\citet{royall1966}; see also Theorem~6.3 in
\citet{DGL} . The argument proceeds in two steps.

\textit{Step 1: consistency of $\hat{a}_0$.}
For any fixed $a_0$, the weighted $k$-NN rule $g_{a_0}$ estimates the conditional probability
$$
    \eta_{a_0}(x) = \mathbb{P}_{a_0}(Y = 0 \mid X = x),
$$
where $\mathbb{P}_{a_0}$ denotes the distribution with class 0 upweighted by $a_0$. By the
standard $k$-NN consistency result, $\hat\eta_{a_0}(x) \to \eta_{a_0}(x)$ in $L_2$ for
almost every $x$. By the law of large numbers,
$$
    \frac{1}{r}\sum_{i=1}^r \mathbf{1}_{\{g_{a_0}(X_i)=0\}}
    \xrightarrow{a.s.} \mathbb{P}_{a_0}(g_{a_0}(X) = 0).
$$
The function $a_0 \mapsto \mathbb{P}_{a_0}(g_{a_0}(X) = 0)$ is continuous and strictly
increasing under the assumption that $\eta_{a_0}(X)$ has a continuous distribution. By
Theorem~\ref{thm:threshold}, $\hat{a}_0 \xrightarrow{P} a_0^\star$, where $a_0^\star$ is the
population-level solution to $\mathbb{P}_{a_0}(g_{a_0}(X) = 0) = b\pi_0$.

\textit{Step 2: consistency of $\hat{g}_b$.}
With $\hat{a}_0 \xrightarrow{P} a_0^\star$ established, the classifier $\hat{g}_b =
g_{\hat{a}_0}$ is a continuous function of $\hat{a}_0$ in the sense that
$\mathbb{P}(g_{\hat{a}_0}(X) \neq g_{a_0^\star}(X)) \to 0$ by the $L_2$ consistency of the
weighted $k$-NN estimator and the continuity of the distribution of $\eta_{a_0^\star}(X)$ at the
decision boundary. Therefore $L(\hat{g}_b) \to L(g_b)$ in probability.
\end{proof}

\begin{remark}
The same argument extends to kernel-weighted $k$-NN rules with weights $w_i(x) =
K(d(x,X_i)/h)$ for a non-increasing kernel $K$, under standard bandwidth conditions.
\end{remark}

\subsection{Support vector machines}
\label{subsec:svm}

The SVM maps inputs $X \in \mathbb{R}^d$ into a reproducing kernel Hilbert space $\mathcal{H}$
via a feature map $\phi$, and seeks a separating hyperplane with maximum margin. Let $\tilde Y_i=2Y_i-1$. The soft-margin SVM solves
$$
    \min_{g \in \mathcal{H},\, a \in \mathbb{R}}
    \frac{1}{2}\|g\|_{\mathcal{H}}^2
    + C \sum_{i=1}^n \max\{0, 1 - \tilde Y_i(g(X_i) + a)\},
$$

yields a real-valued score $$f(x) = \sum_{i=1}^n \alpha_i \tilde Y_i\, k(X_i,x) + a,$$  where $k$ is a positive definite kernel and the coefficients $\alpha_i$ are the solution of the dual optimization problem, satisfying $0\leq \alpha_i\leq C$. Observations with $\alpha_i>0$ are the support vectors and are the only training points that enter the final decision function.  The standard classifier is $g_{\mathrm{svm}}(x) =1_{\{f(x)>0\}}.$

To adapt to the capacity constraint, we propose the modified classifier
$$
    g_b(x) =1_{\{f(x)>\tau\}}, 
$$
where the threshold $\tau$ is estimated on a calibrating sample of size $m$ as
$$\tau=\underset{t\in\mathbb{R}:\, h_t\in\mathcal{C}_b}{\arg \min} \
\widehat L_n(h_t)=\underset{t\in \mathbb{R}:\ \hat{P}(t) \leq b\pi_0 }{\arg \max} \hat{P}_{00}(t)$$
with $\hat{P}(t)=\frac{1}{m}\sum_{i=1}^m \mathbf{1}_{\{h_t(X_i)=-1\}},\quad  \hat{P}_{00}(t)=\frac{1}{m}\sum_{i=1}^m \mathbf{1}_{\{h_t(X_i)=-1, Y_i=1\}},$ and $h_t(x)= 1_{\{f(x)> t \}}$.

\begin{theorem}
\label{thm:svm_consistency}
Assume that the kernel $k$ is a universal kernel (e.g., the Gaussian kernel) and $C = C_n \to \infty$
with $C_n/\sqrt{n} \to 0$. Then the capacity-constrained SVM classifier is consistent:
$L(\hat{g}_b) \to L(g_b)$ in probability.
\end{theorem}

\begin{proof}
The proof is carried out in two steps.

\textit{Step 1: consistency of the score function.}
Under the stated conditions on $C_n$ and the universality of the kernel, the classical SVM is
consistent in the sense that $\mathbb{P}(g_{\mathrm{svm}}(X) \neq g_c(X)) \to 0$
\citep{steinwart2005}. Equivalently, the distribution of the score $f(X)$ converges weakly to
the distribution of the population-level score $f^\star(X)$, where $f^\star$ is the population
minimizer of the regularized risk.

\textit{Step 2: consistency of $\hat\tau$.}
Define $P(t) = \mathbb{P}(f^\star(X) \leq t)$ and $h(t) = b\pi_0 - P(t)$. The threshold $\tau$
is the population-level solution to $P(\tau) = b\pi_0$, i.e., the $(1-b\pi_0)$-quantile of the score distribution. Since $\hat{P}(t) = m^{-1}\sum_i \mathbf{1}_{\{f(X_i) \leq t\}}$ and the
empirical distribution of $f(X_1), \ldots, f(X_m)$ converges uniformly to $P(t)$ by the Glivenko--Cantelli theorem, we have $\sup_t |\hat{P}(t) - P(t)| \xrightarrow{a.s.} 0$.
Provided $P$ is continuous and strictly increasing at $\tau$ (i.e., the score distribution has no atom at $\tau$), Theorem~\ref{thm:threshold} gives $\hat\tau \xrightarrow{P} \tau$.

\textit{Step 3: conclusion.}
Since $\hat\tau \xrightarrow{P} \tau$ and $\mathbb{P}(f(X) = \tau) = 0$ by the continuity of the
score distribution, the classifier  $1_{\{f(x)>\hat \tau\}}$ converges in probability to
 $1_{\{f^\star(x)> \tau\}}= g_b(X)$ for almost every $X$. Therefore $L(\hat{g}_b) \to
L(g_b)$ in probability by dominated convergence.
\end{proof}

\subsection{Random forests}
\label{subsec:rf}

In the same spirit as the approach adopted for SVM we propose to decouple score estimation from threshold selection. Rather than modifying the underlying learning algorithm, we first construct a real-valued scoring function and subsequently determine the decision threshold so as to enforce the desired capacity constraint. This two-step strategy preserves the predictive strength of the base classifier while introducing explicit control over the global proportion of positive predictions.
		
Let $\hat\eta_{\mathrm{RF}}(x) = B^{-1}\sum_{j=1}^B g_{Z_j}(x)$ denote the Random Forest estimate of
$\eta(x)=\mathbb{P}(Y=1\mid X=x)$,
where $B$ is the total number of trees in the forest and each tree $g_{Z_j}$ returns the empirical probability of class $1$ in the terminal leaf containing $x$. The capacity-constrained Random Forest classifier is
$$
    \hat{g}_b(x) = 1_{\{\hat\eta_{\mathrm{RF}}(x) - \tfrac{1}{2} > \tau \}},$$
 where $\hat\tau$ is estimated on a calibration sample analogously to the SVM case:
$$
    \hat\tau = \underset{t\in\mathbb{R}:\, s_t\in\mathcal{C}_b}{\arg \min} \
\widehat L_n(s_t)=\operatorname*{arg\,max}_{t:\, \hat{P}(t) \leq b\pi_0} \hat{P}_{00}(t).
$$

		with
		$\hat{P}(t)=\frac{1}{m}\sum_{i=1}^m \mathbf{1}_{\{s_t(X_i)=-1\}},\quad  \hat{P}_{00}(t)=\frac{1}{m}\sum_{i=1}^m \mathbf{1}_{\{s_t(X_i)=-1, Y_i=1\}},$
		and $s_t(x)=   1_{\{\frac{1}{B}\sum_{j=1}^{B} g_{Z_j}(x)-1/2 > t\}}$. The estimation of $\tau$ is performed using an independent training set of size $m$.
		
		We now establish the consistency of the  classifier under standard assumptions.

\begin{theorem}
\label{thm:rf_consistency}
Assume:
\begin{enumerate}
    \item $\hat\eta_{\mathrm{RF}}(x) \to \eta(x) := \mathbb{P}(Y=1 \mid X=x)$ in probability
    for almost every $x$;
    \item The distribution of $\eta(X)$ is continuous at $\tfrac{1}{2} + \tau$;
    \item $\tau$ is the unique solution to $\mathbb{P}(\eta(X) \geq \tfrac{1}{2} + \tau) =
    b\pi_0$.
\end{enumerate}
Then: (i) $\hat\tau \xrightarrow{P} \tau$; (ii) $\mathbb{P}(\hat{g}_b(X) \neq g_b^\star(X)) \to
0$, where $g_b^\star(x) = \mathrm{sign}(\eta(x) - \tfrac{1}{2} - \tau)$.
\end{theorem}

\begin{proof}
\textit{Step 1: convergence of $\hat\tau$.}
By assumption (1) and the continuous mapping theorem, the empirical distribution of
$\hat\eta_{\mathrm{RF}}(X_1), \ldots, \hat\eta_{\mathrm{RF}}(X_m)$ converges weakly to the
distribution of $\eta(X)$. Define $P(t) = \mathbb{P}(\eta(X) \geq \tfrac{1}{2} + t)$ and
$\hat{P}(t) = m^{-1}\sum_i \mathbf{1}_{\{\hat\eta_{\mathrm{RF}}(X_i) \geq 1/2 + t\}}$. By the
Glivenko--Cantelli theorem applied to the estimated scores, $\sup_t |\hat{P}(t) - P(t)|
\xrightarrow{P} 0$. By assumption (2), $P$ is continuous at $\tau$, and by assumption (3),
$\tau$ is the unique solution to $P(\tau) = b\pi_0$. Theorem~\ref{thm:threshold} then gives
$\hat\tau \xrightarrow{P} \tau$.

\textit{Step 2: convergence of $\hat{g}_b$.}
Fix any $x$ such that $\eta(x) \neq \tfrac{1}{2} + \tau$. Since $\hat\eta_{\mathrm{RF}}(x)
\xrightarrow{P} \eta(x)$ and $\hat\tau \xrightarrow{P} \tau$, we have $\hat\eta_{\mathrm{RF}}(x)
- \tfrac{1}{2} - \hat\tau \xrightarrow{P} \eta(x) - \tfrac{1}{2} - \tau \neq 0$. 

The indicator function is continuous at every point except at the threshold $\tau$,  so $\hat{g}_b(x) \xrightarrow{P} g_b^\star(x)$
for almost every $x$. By dominated convergence, $\mathbb{P}(\hat{g}_b(X) \neq g_b^\star(X)) \to
0$.
\end{proof}

\subsection{Neural networks}
\label{subsec:nn}

Let $\mathcal{C}_k$ denote the class of neural networks with one hidden layer and $k$ neurons,
where each $g \in \mathcal{C}_k$ is of the form $g(x) = \mathbf{1}_{\{\psi(x) \geq 1/2\}}$
with
\[
    \psi(x) = c_0 + \sum_{i=1}^k c_i \sigma(\psi_i(x)),
    \qquad
    \psi_i(x) = b_i + \sum_{j=1}^d a_{ij} x_j,
\]
and $\sigma$ is the threshold sigmoid with values in $\{-1, 1\}$. Consistency in this setting
follows from Corollary~\ref{cor:erm} together with the following VC dimension bound.

\begin{theorem}[\citealt{baum}]

\label{thm:baum_haussler}
The shatter coefficients of $\mathcal{C}_k$ satisfy
\[
    S(\mathcal{C}_k, n) \leq (ne)^{kd + 2k + 1},
\]
and therefore
\[
    \mathrm{VCdim}(\mathcal{C}_k) \leq 2(kd + 2k + 1)\log(e(kd + 2k + 1)).
\]
\end{theorem}

\begin{corollary}
\label{cor:nn_consistency}
Let $k = k_n \to \infty$ with $k_n \log(n)/n \to 0$. Then:
\begin{enumerate}
    \item The empirical risk minimizer $g_n$ over $\mathcal{C}_k$ satisfies $R(g_n) \to R^\star_{\mathcal{C}} := \inf_{g \in \mathcal{C}} R(g)$ almost surely for all distributions of $(X,Y)$.
    \item The empirical capacity-constrained minimizer $\hat{g}_b$ over $\mathcal{C}_{b,k} :=
    \mathcal{C}_k \cap \mathcal{C}_b$ satisfies $L(\hat{g}_b) \to L^\star_{\mathcal{C}_b} :=
    L(g_b)$ almost surely.
\end{enumerate}
\end{corollary}

\begin{proof}
Part (1) is the classical result of \citet{baum1}; see Theorem~30.6 in \citet{DGL}.
For Part (2), by Theorem~\ref{thm:vc}, $\mathrm{VCdim}(\mathcal{C}_{b,k}) \leq
\mathrm{VCdim}(\mathcal{C}_k) \leq 2(kd+2k+1)\log(e(kd+2k+1))$. Under the condition $k_n
\log(n)/n \to 0$, the VC inequality gives
\[
    \sup_{g \in \mathcal{C}_{b,k}} |\hat{L}_n(g) - L(g)| \xrightarrow{a.s.} 0.
\]
The conclusion then follows from Corollary~\ref{cor:erm}.
\end{proof}

\subsection{Empirical Evaluation}
\label{subsec:setup_emp}
We compare three families of classifiers. The classical classifier $g$ is trained without any capacity constraint and uses the default decision threshold. The SMOTE-based classifier $g_{\mathrm{smote}}$ is trained on an augmented sample in which the minority class is oversampled by a factor of four. To assess robustness, we also consider a more aggressive non-linear oversampling scheme with amplification factors up to 20 in Appendix~A. Finally, the proposed capacity-constrained classifier $g_b$ is evaluated for $b\in\{1,2,3\}$.

Each family is instantiated with three base learners: Random Forest (RF), Support Vector Machine (SVM, Gaussian kernel), and $k$-nearest neighbors (kernel-weighted, Gaussian kernel). For RF and SVM, the capacity constraint is enforced by adjusting the threshold on a left-over sample, as described in Sections~\ref{subsec:svm} and~\ref{subsec:rf}. For $k$-NN, the weight $a_0$ is calibrated as described in Section~\ref{subsec:knn}.

The empirical evaluation is based on training samples of size 8{,}000. The minority-class proportion $\pi_0=\mathbb{P}(Y=0)$ is varied by resampling, allowing us to assess performance across controlled imbalance regimes. 

For the proposed classifier $g_b$, each training sample is divided into two equal parts. The first part is used to fit the underlying model, while the second is used to calibrate the capacity-related parameter: $a_0$ for $k$-NN and $\tau$ for RF and SVM. No additional hyperparameters are introduced to control this split proportion.

Performance is estimated on a separate test sample of size 10{,}000. The entire procedure is repeated 100 times, except for SVM, for which 30 repetitions are used.  The reported results correspond to averages over replications.

For the purpose of illustrating the proposed methodology, we deliberately avoid introducing additional hyperparameter tuning in the baseline classifiers. The goal of this empirical study is not to obtain the best possible predictive model for this particular dataset, but rather to show how standard classification methods can be combined with the proposed capacity-based calibration procedure.

The baseline RF classifier was trained with 503 trees, using randomly selected candidate variables $\lfloor \sqrt{p} \rfloor$ at each split, where $p=23$ is the number of predictors. The minimum terminal node size was set to 5. No class weights were used. SVM classifiers were trained using the standard $C$-classification formulation with a radial basis function kernel, with $C=1$ and $\gamma=0.1$. The weighted $k$-NN classifier was implemented with $k=5$ and a Gaussian kernel. Thus, predictions were based on the five nearest neighbors, with closer observations receiving larger weights.

The empirical analysis is based on the Taiwanese credit card default dataset introduced by~\citet{credit}, which has been widely used to evaluate classification methods in imbalanced settings. The data consist of payment records for 25{,}000 credit card clients collected in October 2005 from a major commercial bank in Taiwan, of whom 5{,}529 (22.1\%) defaulted in the following month. Since the original dataset exhibits only moderate imbalance, we construct controlled imbalance scenarios by resampling minority observations to achieve target values of $\pi_0 \in [0.005, 0.05]$.
	\begin{figure}
		\centering
		\includegraphics[width=1.\textwidth]{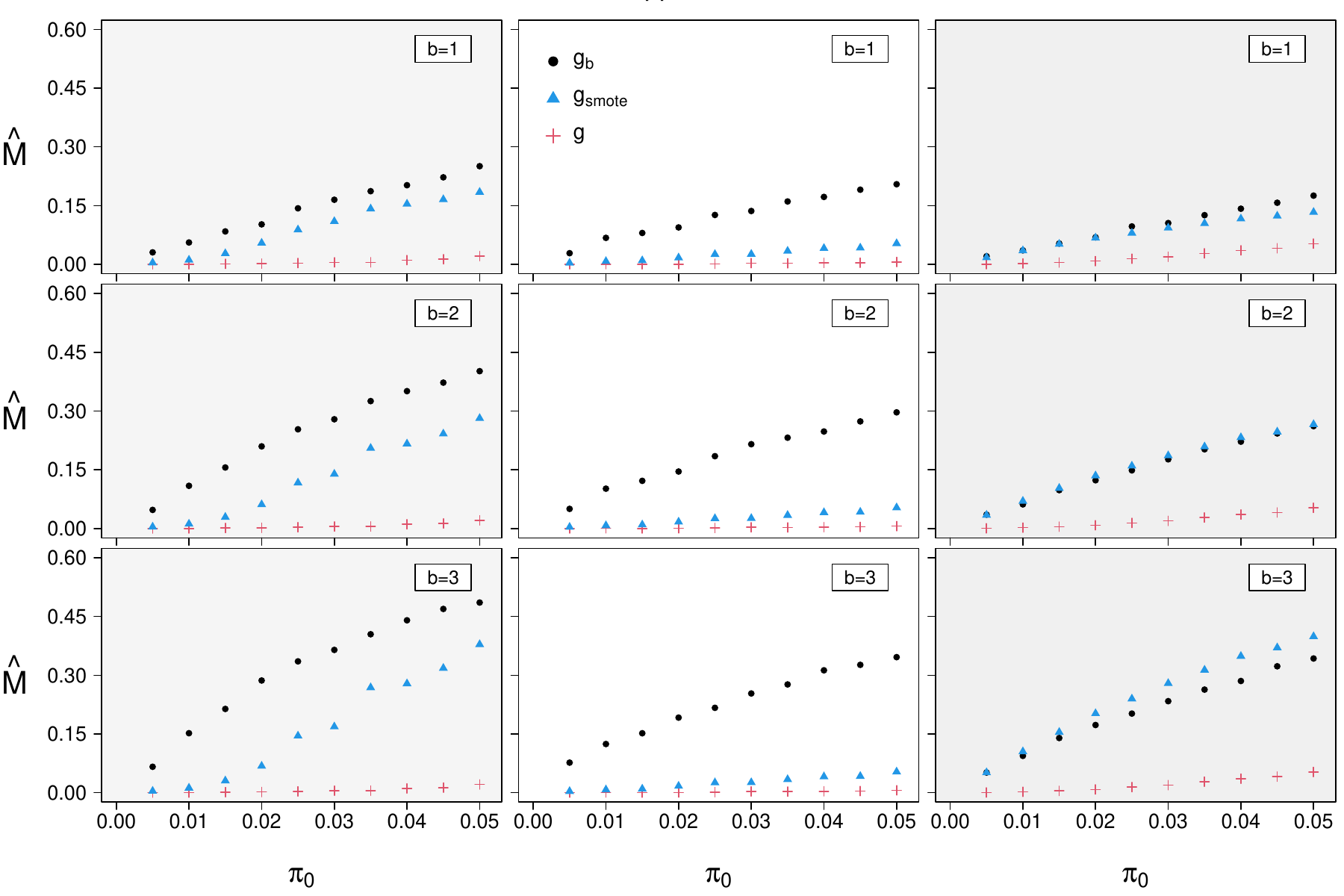}
		\caption{Comparison of classical classifiers, Random Forest, SVM, and k-NN (red crosses), with our capacity-constrained adaptation for  $b=1,2,3$ (black points), along with classical classification under SMOTE (blue triangles). Columns represent methods and rows represent values of 
$b$.  Each panel shows $\hat{M}$ as a function of $\pi_0$ }
		\label{fig_datos}
	\end{figure}
Figure~\ref{fig_datos} summarizes the results. The three columns correspond to RF, SVM, and $k$-NN; the rows correspond to $b \in \{1, 2, 3\}$. Each panel plots $\hat{M}$ as a function of $\pi_0$ for the classical classifier (red crosses), SMOTE (blue triangles), and the proposed capacity-constrained classifier (black circles).

Across the experiments, classical classifiers exhibit very low values of $\hat{M}$, especially for small values of $\pi_0$. This behavior reflects their tendency to predict the majority class when the minority class is rare, leading to high overall accuracy but near-zero sensitivity to the target class. For instance, at $\pi_0=0.02$, the unconstrained Random Forest achieves only $\hat{M}\approx 0.0016$, which corresponds to detecting fewer than one default per thousand cases under the capacity constraint.

The proposed capacity-constrained classifiers achieve substantially higher values of $\hat{M}$ throughout the experiments, with improvements becoming more pronounced as $b$ increases. Larger values of $b$ allow a higher proportion of observations to be flagged as belonging to the minority class, increasing coverage at the cost of reduced specificity. For example, at $\pi_0=0.02$ and $b=3$, the capacity-constrained Random Forest achieves $\hat{M}\approx 0.30$, representing a 187-fold improvement over the unconstrained classifier.

To interpret this concretely, suppose a bank has 10{,}000 credit card clients to assess, of whom 200 ($\pi_0=0.02$) are expected to default. If the bank can subject at most $b\pi_0 \times 10{,}000 = 600$ clients to exhaustive review, the proposed method identifies approximately 60 future defaults, whereas the unconstrained classifier identifies fewer than one.

The magnitude of the improvement varies across classification methods. All three methods benefit markedly from the capacity constraint, but Random Forest consistently achieves the highest values of $\hat{M}$ across the settings considered. This behavior is consistent with the strong empirical performance of Random Forest on tabular data.

The effect of SMOTE is method-dependent. For SVM, oversampling yields negligible improvements, with performance remaining close to zero. For Random Forest, SMOTE improves over the unconstrained baseline but remains below the capacity-constrained approach. For $k$-NN, SMOTE is sometimes comparable to the proposed method at lower imbalance levels, but this advantage disappears when $\pi_0 \lesssim 0.02$. Overall, these results suggest that SMOTE is not directly aligned with the target metric $\hat{M}$, since it modifies the training distribution without explicitly linking the oversampling level to the operational constraint.

These results show the potential gains of capacity-constrained classification. A comparison with the simpler post-hoc thresholding baseline is provided in Section~\ref{sec:beyond}.


\section{A more efficient alternative: Capacity-Aware Learning}
\label{sec:beyond}

The capacity-constrained classifiers of Section~\ref{sec:consistency} 
improve upon standard classification by orienting the choice of 
hyperparameters toward $M$ subject to the capacity constraint. 
To understand the limits of this approach, we first compare it with 
a simpler baseline, and then introduce a method that goes beyond it.

\subsection{Comparison with post-hoc thresholding}
\label{subsec:posthoc}
A natural baseline that has not yet been discussed is to leave the 
learning algorithm entirely unchanged and enforce the capacity 
constraint only at the prediction stage. Given any classifier that 
produces a real-valued score $s$, one ranks all observations by score 
and selects a threshold so that at most a fraction $b\pi_0$ of the 
population is flagged. This post-hoc thresholded classifier is simple, 
universally applicable, and satisfies the capacity constraint by 
construction. 

 Assume, without loss of generality, that larger values of $s(x)$ provide stronger evidence in favor of class $0$. Define the threshold,
$$
\tau(s) \;=\; \inf\bigl\{\tau \in \mathbb{R} : \Pr{s(X) \geq \tau} 
\leq b\pi_0 \bigr\},
$$
and the corresponding post-hoc thresholding classifier
$$
g_{\mathrm{p\text{-}h}}(x) \;=\;
\begin{cases}
0 & \text{if } s(x) \  \geq \ \tau(s), \\
1 & \text{otherwise.}
\end{cases}
$$

By construction,
$$
\mathbb{P}\bigl(g_{\mathrm{p\text{-}h}}(X)=0\bigr)\leq b\pi_0.
$$
The following result characterizes when a capacity-constrained classifier reduces to post-hoc thresholding.

\begin{proposition}\label{th:ph}
Let $g$ be a capacity-constrained classifier whose score function $s$ 
has no free parameters. Then $g$ coincides with the post-hoc 
thresholding classifier $g_{\mathrm{p\text{-}h}}$ based on $s$.
\end{proposition}

\begin{proof}
Since $s$ has no free parameters, the only degree of freedom in 
constructing $g$ is the threshold $\tau$, which is determined 
entirely by the capacity constraint $P(g(X) = 1) \leq b\pi_0$. 
This is precisely the definition of $\tau(s)$, so $g = 
g_{\mathrm{p\text{-}h}}$.
\end{proof}
  
The distinction between post-hoc thresholding and capacity-constrained learning arises when the score itself depends on tuning parameters. In the capacity-constrained procedures introduced above, one considers a family of scores $s_\theta$, where $\theta$ is selected to maximize $P_{0,0}(g_{\theta,b}) = \mathbb{P}(g_{\theta,b}(X)=0\mid Y=0)$ subject to $\mathbb{P}(g_{\theta,b}(X)=0)\leq b\pi_0.$ Thus, the capacity constraint is used not only to choose the final threshold, but also to shape the score function.

As an illustrative example, consider the $k$-NN method. For a fixed number of neighbors $k$,
the score is the estimated conditional probability
$$
\hat{\eta}_k(x)
=
\hat{\mathbb{P}}(Y=0\mid X=x).
$$
The post-hoc thresholding classifier keeps this score fixed and only adjusts the final threshold
to satisfy the capacity constraint:
$$
g_{\mathrm{p\text{-}h}}(x)
=
\begin{cases}
0, & \text{if } \hat{\eta}_k(x)\geq q_{1-b\pi_0},\\
1, & \text{otherwise,}
\end{cases}
$$
where $q_{1-b\pi_0}$ denotes the empirical $(1-b\pi_0)$-quantile of
$\hat{\eta}_k(x_1),\ldots,\hat{\eta}_k(x_n)$.

We compare this post-hoc classifier with the capacity-constrained $k$-NN classifier introduced in
Section~3.2.1. In that procedure, both the number of neighbors $k$ and the threshold parameter
$a_0$ are treated as free parameters. The classifier is selected by minimizing the empirical
loss over the class of admissible capacity-constrained classifiers:
$$
\hat g_b
=
\arg\min_{g\in\mathcal C_b}
\hat L_n(g),
$$
where
$$
\mathcal C_b
=
\left\{
g_{k,a_0}\in\mathcal C:\;
\mathbb{P}\bigl(g_{k,a_0}(X)=0\bigr)\leq b\pi_0,\;
k\in\mathbb N,\;
a_0\in(0,1)
\right\}.
$$
Thus, while $g_{\mathrm{p\text{-}h}}$ applies the capacity constraint only after the $k$-NN score has been constructed, $\hat g_b$ uses the capacity constraint to choose the classifier itself.

In practice, all thresholds are replaced by their empirical counterparts. Thus, the threshold is chosen so that no more than $\lfloor b\pi_0 n\rfloor$ observations in the calibration sample are assigned to class 0. The same empirical calibration is used both for post-hoc thresholding and for the capacity-constrained classifiers after the parameter $\hat{\theta}$ has been selected. The important point is that, in the post-hoc procedure, the value of  $k$ is chosen independently of the capacity-constrained objective.  A classifier that performs well in terms of overall accuracy may still induce a poor ordering of the observations that matter most under the constraint.  

Figure~\ref{fig_gain} quantifies the percentage gain of $\hat g_b$ over $g_{\mathrm{p\text{-}h}}$ on the credit card default dataset, as a function of $\pi_0$ and for $b\in\{1,2,3\}$. The gain is positive across the range of imbalance levels considered, showing that the capacity-constrained approach improves upon the post-hoc baseline. For this dataset, the largest improvements are observed at smaller values of $\pi_0$, where class imbalance is most severe. As $\pi_0$ increases, the gain decreases and the results for the three capacity levels tend to converge.
	\begin{figure}
		\centering
\includegraphics[width=0.6\textwidth]{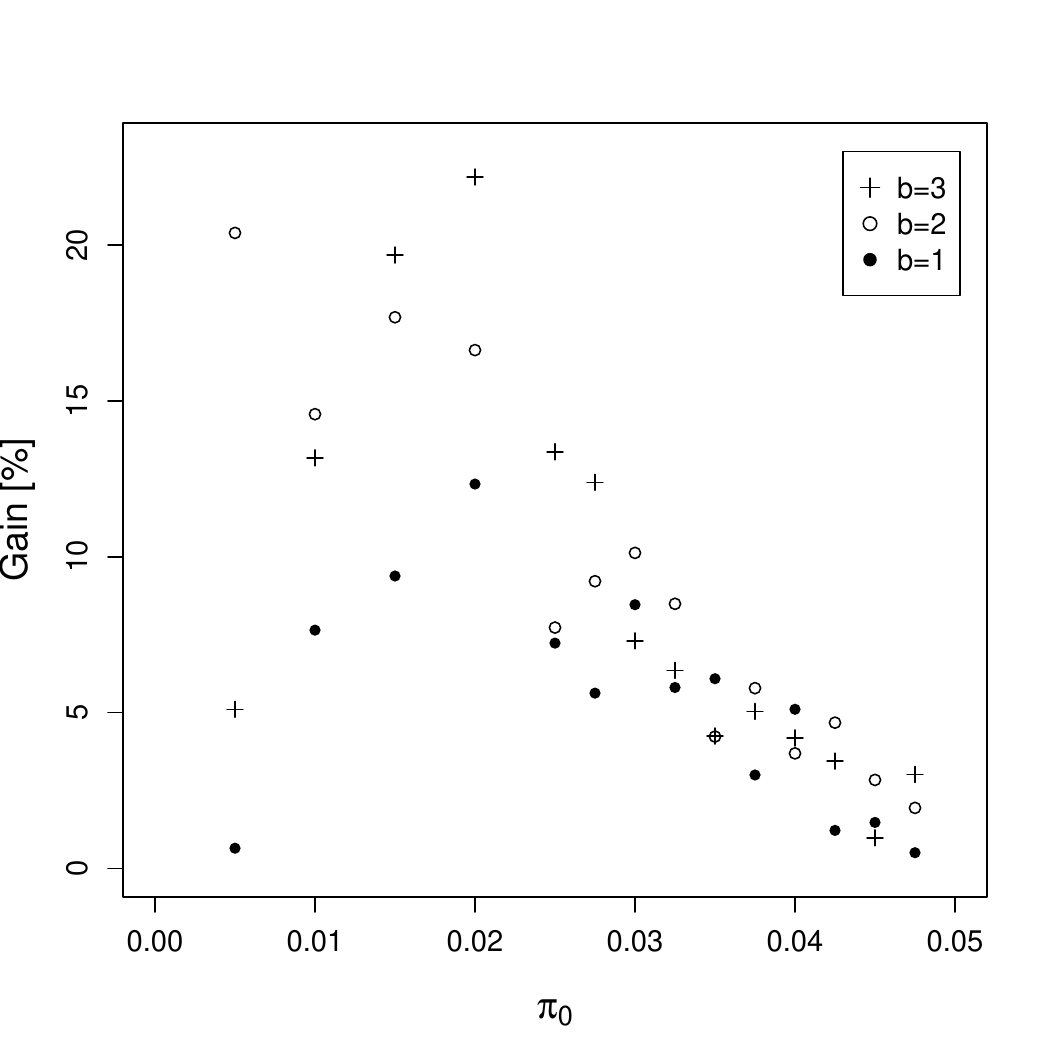}
		\caption{Percentage gain of the capacity-constrained weighted $k$-NN classifier over the post-hoc
thresholding baseline as a function of $\pi_0$, for different values of $b$.   The post-hoc
baseline thresholds the weighted $k$-NN estimate $\widehat{\eta}_k(x)$, computed with Gaussian nearest-neighbor weights $w_j(x)$.   In the implementation used by the \texttt{kknn} package with
\texttt{kernel = "gaussian"}, the Gaussian kernel is applied to internally rescaled nearest-neighbor distances. For the standard $k$-NN benchmark, the number of nearest neighbors $k$ is tuned by cross-validation using classification accuracy. Positive values indicate improved performance over the post-hoc baseline.        }
		\label{fig_gain}
	\end{figure}
    
    Post-hoc thresholding remains useful when retraining is not feasible, for instance when the
practitioner only has access to a pre-trained model. In such cases, $g_{\mathrm{p\text{-}h}}$ provides a simple and principled way to enforce the capacity constraint.

The procedures presented in Sections~3.3--3.5 go one step further. They follow a two-stage strategy: first, a score function is estimated from the data; then, a threshold is selected to satisfy the capacity constraint. This approach is attractive because it can be applied to any classifier producing a real-valued output, without modifying the underlying learning algorithm. However, the decoupling between score estimation and threshold selection may be suboptimal: the score is still trained under the original unconstrained objective and may not rank observations in the order most useful for the capacity-constrained problem.

A more ambitious approach is to incorporate the capacity constraint directly into the training
stage, so that the learned classifier is shaped by the resource limitation from the outset. We refer to this strategy as capacity-aware learning. In what follows, we present a capacity-aware method for support vector machines. Similar ideas could also be developed for random forests and other learning methods.

\subsection{Capacity-Aware Support Vector Selection} \label{subsec:awaresvm}
Classical support vector machines determine the separating hyperplane using all observations associated with nonzero dual coefficients ($\alpha_i$).  More precisely, for binary labels $y_i\in\{0,1\}$, let $\tilde y_i=2y_i-1\in\{-1,1\}$. The SVM score can then be written as
$$
f(x)=\sum_{i=1}^n \alpha_i \tilde{y}_i K(x_i,x)+a,
$$
the support vectors are precisely those observations satisfying $\alpha_i>0.$ This selection is entirely driven by geometric margin considerations. However, in imbalanced classification problems under operational constraints, maximizing the margin does not necessarily coincide with the final decision objective. In particular, when the goal is to maximize the detection probability of a target class under a capacity constraint, different support vectors may have very different impacts on the constrained performance measure.

Motivated by this observation, we propose a capacity-aware support vector selection procedure. Instead of using all support vectors returned by the standard SVM solution, we construct a reduced subset of support vectors chosen according to their contribution to the constrained classification objective.

Let
$
SV=\{i:\alpha_i>0\}
$
denote the set of support vectors obtained from a standard SVM fit. For any subset
$
S\subseteq SV,
$
define the restricted score function
$$
f_S(x)=\sum_{i\in S}\alpha_i \tilde y_i K(x_i,x)+a.
$$
Given a capacity level $b\pi_0$, we define the associated threshold
$\tau_b(S)$
as the empirical quantile satisfying
$$
\mathbb{P}(f_S(X)\leq \tau_b(S))\leq b\pi_0.
$$

The resulting constrained classifier is
$$
g_{S,b}(x)=
\mathbf{1}\{f_S(x)\geq \tau_b(S)\}.
$$
The proposed methodology searches for subsets of support vectors maximizing the detection probability of the target class under the capacity constrain:
$$
S^\star
=
\arg\max_{S\subseteq SV}
\mathbb{P}(g_{S,b}(X)=0\mid Y=0).
$$

The resulting constrained SVM classifier is then defined by
$$
g^\star_{b}=g_{S^\star,b}.
$$

Since exhaustive search is computationally infeasible when the number of support vectors is large, we employ a greedy forward-selection strategy.

Starting from the empty set,
$$
S_0=\emptyset,
$$
at each iteration the algorithm adds the support vector producing the largest increase in constrained performance:
$$
i^\star
=
\arg\max_{i\in SV\setminus S_t}
M_b(g_{S_t\cup\{i\},b}).
$$

The selected support vector is then added to the active subset:
$$
S_{t+1}=S_t\cup\{i^\star\}.
$$
The procedure is repeated until one of two stopping criteria is met: either the selected subset
contains $n_{\max}$ support vectors, or adding any remaining support vector fails to improve the constrained performance criterion.

Unlike classical SVM optimization, where support vectors are selected solely according to margin geometry, the proposed approach selects support vectors according to their contribution to constrained minority-class detection. Consequently, the resulting classifier is explicitly adapted to operational constraints and class imbalance.

This framework can be interpreted as a bridge between statistical classification and constrained resource allocation, where support vectors are selected not only for geometric separation but also for their relevance to the final decision objective.

\subsection{Empirical Illustration}
\label{sec:empirical}

The goal of this section is not to conduct an exhaustive benchmark of classification methods, but rather to illustrate the theoretical hierarchy established in Sections~\ref{sec:consistency}
and~\ref{subsec:posthoc} on real data. Specifically, we aim to show 
that the gains from capacity-constrained classification are 
empirically meaningful, and that the method introduced in 
Section~\ref{sec:beyond} extracts additional value beyond what can 
be achieved by threshold adjustment alone.

To this end, we consider three variants of a support vector machine 
classifier. In all the variants, the kernel and regularization 
parameter are fixed in advance and not optimized with respect to 
$P_{00}$, so that by Theorem~\ref{th:ph} the 
capacity-constrained classifier reduces to post-hoc thresholding ($g_b=g_{\mathrm{p\text{-}h}}$).  This is deliberate: it allows us to isolate the contribution of the threshold adjustment alone, before introducing the additional degrees of freedom that characterize the third variant. The three variants are defined as follows. The standard SVM classifier, denoted by $g$, is applied directly without enforcing the capacity constraint. The classifier $g_b$ is the capacity-constrained SVM of Section~3.4. By Theorem~\ref{th:ph}, this coincides with the post-hoc thresholding classifier $g_{\mathrm{p\text{-}h}}$ in this setting, since the threshold $\tau$ is the only degree of freedom available to enforce the capacity constraint. Finally, $g_b^\star$ denotes the capacity-aware SVM introduced in Subsection~\ref{subsec:awaresvm}, which incorporates the capacity constraint directly into the training stage and can improve upon $g_b$ by exploiting additional degrees of freedom.

\begin{figure}
		\centering
\includegraphics[width=0.94\textwidth]{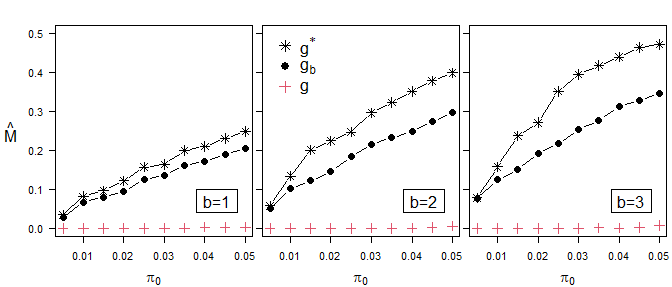}
	\caption{
Performance of the standard SVM classifier $g$, the capacity-constrained SVM $g_b$, and the capacity-aware SVM $g_b^\star$ as a function of $\pi_0$, for $b=1,2,3$. 
}
\label{fig_svm_capacity_aware}
	\end{figure}
For each method, we report $\hat{M}$ at three representative values of $b$. 
All results are averaged over 30 random realizations. In the support-vector selection procedure,
we set $n_{\max}=30$, while the total number of support vectors in the training stage is typically of the order of 100. We considered two variants: one using all 30 support vectors selected by the
greedy procedure, and another selecting the subset along the greedy path that maximizes $\hat M$ on the validation set. The results shown here correspond to the first variant; the second yielded similar conclusions.

Figure~\ref{fig_svm_capacity_aware} shows that the ordering
$$g \leq g_b \leq g_b^\star
$$
holds across the considered imbalance and capacity levels. The standard SVM classifier $g$ remains close to zero, while the threshold-calibrated classifier $g_b$ improves substantially. The capacity-aware classifier $g_b^\star$, based on support-vector selection, achieves the highest values of $\hat M$ in most settings, with larger gains as $b$ increases.
Overall, the results provide empirical evidence that capacity-aware learning can improve upon post-hoc calibration by reshaping the SVM score itself.

\color{black}

		\section{Extensions}
\label{sec:extensions}

We discuss two natural extensions of the framework developed in
Sections~\ref{sec:theory}--\ref{sec:empirical}: multiclass classification
(Section~\ref{subsec:multiclass}), and online classification under capacity constraints
(Section~\ref{subsec:online}).

\subsection{Multiclass classification}
\label{subsec:multiclass}

The binary framework extends naturally to settings with $M \geq 2$ classes. Let $(X, Y) \in
\mathbb{R}^d \times \{0, 1, \ldots, M-1\}$ and let $g : \mathbb{R}^d \to \{0, 1, \ldots, M-1\}$
be a classifier. In the classical setup, the optimal classifier minimizes the misclassification
probability $\mathbb{P}(g(X) \neq Y)$.

When only one class, say class 0 is rare and subject to a capacity constraint, the framework
of Section~\ref{sec:theory} carries over directly. In this case, the admissible class is defined as
$$
    \mathcal{C}_b
    :=
    \left\{
    g : \mathbb{P}(g(X)=0) \leq b\pi_0
    \right\}.
$$
The objective is then to choose, within $\mathcal{C}_b$, the classifier that maximizes the
detection probability of the rare class,
$$
    g^\star = \operatorname*{arg\,max}_{g \in \mathcal{C}_b}
    \mathbb{P}(g(X) = 0 \mid Y = 0).
$$
The optimal solution again thresholds the likelihood ratio $f_0(x) / \sum_{j \neq 0} \pi_j
f_j(x)$ at a value $\gamma^\star$ chosen to satisfy the capacity constraint, by an argument
identical to the proof of Proposition~\ref{prop:optimal}.

When several classes are rare, say classes 0 and 1, one may define joint capacity constraints that limit the total proportion of observations assigned to the rare classes. For example, a classifier \(g\) may be required to belong to the class
$$
    \mathcal{C}_{b_0, b_1} := \{g :
    \mathbb{P}(g(X) = 0) \leq b_0 \pi_0 \text{ and }
    \mathbb{P}(g(X) = 1) \leq b_1 \pi_1\},
$$
and optimize the weighted sum of per-class sensitivities,
$$
    g^\star = \operatorname*{arg\,max}_{g \in \mathcal{C}_{b_0, b_1}}
    \Bigl[
        \mathbb{P}(g(X) = 0 \mid Y = 0)
        + \mathbb{P}(g(X) = 1 \mid Y = 1)
        + \mathbb{P}(g(X) = Y \mid Y \in \{2, \ldots, M-1\})
    \Bigr].
$$
This formulation assigns equal weight to the detection of each imbalanced class and to the
correct classification of the majority block. The capacity parameters $b_0$ and $b_1$ can be
set independently, reflecting different operational costs for each minority class. The
consistency results of Section~\ref{sec:consistency} extend to this setting under the same
conditions, since $\mathcal{C}_{b_0, b_1} \subseteq \mathcal{C}$ and
Theorem~\ref{thm:vc} applies to each constraint separately.

\subsection{Online classification}
\label{subsec:online}

In many applications, observations arrive sequentially in real time and decisions must be made
before the next observation is received. Fraud detection and medical triage are canonical
examples. The capacity constraint takes a different form in this setting: rather than bounding
the global proportion of positive predictions, one typically imposes a constraint on the
expected rate of positive predictions per unit time, i.e., $\mathbb{E}[\mathbf{1}_{g(X_t) = 0}]
\leq b\pi_0$ for each time step $t$.

The threshold-based classifiers developed in Sections~\ref{subsec:svm}
and~\ref{subsec:rf} adapt naturally to this setting. The threshold $\tau_t$ can be updated
sequentially as new observations arrive, using a sliding window or an exponential moving average
of the recent flagging rate:
\[
    \hat{P}_t = (1 - \alpha)\hat{P}_{t-1} + \alpha \mathbf{1}_{\{g(X_t) = 0\}},
\]
for a smoothing parameter $\alpha \in (0,1)$. The threshold is adjusted upward when
$\hat{P}_t > b\pi_0$ and downward otherwise. This procedure requires no retraining of the
underlying model and adapts to slow distributional shifts, such as changes in the base rate
$\pi_0$ over time.

A full theoretical treatment of the online setting, including regret bounds and consistency
under non-stationary distributions, is beyond the scope of this paper and is left for future
work. We note, however, that the time-varying parameter $b_t$ case,where the available
capacity itself fluctuates due to analyst availability or system load, is a particularly
natural and practically relevant direction.

\section{Conclusions}
\label{sec:conclusions}

We have introduced a formal framework for classification under capacity constraints, motivated by the observation that in many real-world applications the relevant question is not how to minimize overall error, but how to detect as many minority-class instances as possible given a fixed budget of positive predictions. This is a distinct problem from standard imbalanced classification, and it calls for a distinct solution.

The central theoretical contribution is the characterization of the 
optimal capacity-constrained classifier (Proposition~\ref{prop:optimal}) 
and its equivalence with the classical Bayes classifier under a 
modified prior (Proposition~\ref{prop:bayes_equiv}). This equivalence 
reveals that the capacity parameter $b$ plays the role of an implicit 
prior adjustment: it amplifies the minority class to a degree 
determined not by an arbitrary oversampling factor, but by the 
operational constraint itself. Unlike SMOTE and related 
data-augmentation methods, the reweighting is transparent, grounded 
in observable quantities, and does not distort the geometry of the 
minority class.

The statistical contribution is the establishment of consistency for 
five families of learning procedures --- kernel density estimators, 
$k$-nearest neighbors, support vector machines, random forests, and 
neural networks --- under the capacity constraint. A unifying insight 
is that the capacity constraint reduces the effective complexity of 
the classifier class (Theorem~\ref{thm:vc}), so that any consistent 
procedure over the unconstrained class remains consistent over the 
constrained class. Full proofs are provided in each case, replacing 
the proof sketches that typically appear in the literature for 
related results.

A further contribution is the clarification of the relationship between capacity-constrained learning and post-hoc thresholding. We show (Theorem~\ref{th:ph}) that the two procedures coincide whenever the score function has no free parameters oriented toward the capacity-constrained objective: in that case, adjusting the threshold is the only available degree of freedom, and the 
resulting classifier is equivalent to post-hoc thresholding. This result identifies precisely where the gains from capacity-constrained learning originate, not from threshold adjustment per se, but from fitting the model with the right objective from the start.

Building on this insight, we introduce capacity-aware learning as a  more ambitious approach that incorporates the capacity constraint  directly into the training stage. As a concrete instance, we develop a capacity-aware support vector machine that shapes the decision boundary itself according to the operational constraint. The empirical contribution is the introduction of the capacity-adjusted detection rate $M$ as the appropriate performance metric in this setting, and its systematic evaluation across methods and datasets. 
The experiments demonstrate that capacity-constrained classifiers substantially outperform both classical approaches and SMOTE under high imbalance regimes. A direct comparison among crude classification, post-hoc thresholding, and capacity-aware learning confirms the theoretical hierarchy: each level of the framework yields measurable gains, with the capacity-aware SVM achieving the strongest performance.

Several directions remain open. On the theoretical side, the 
analysis of capacity-aware learning could be extended to other 
method families, particularly gradient-based methods where the 
capacity constraint could be incorporated through a Lagrangian 
relaxation. Establishing consistency and convergence rates for the 
capacity-aware SVM introduced here is also a natural next step. On 
the applied side, the extension to settings with random or 
time-varying capacity $b_t$, arising from fluctuations in analyst 
availability or system load, is both practically relevant and 
theoretically non-trivial. The online setting, where decisions must 
be made sequentially under a rolling capacity constraint, presents 
additional challenges that merit a dedicated treatment. Another interesting direction for future work is to extend this approach to reinforcement learning settings.

Taken together, the theoretical and empirical results of this paper  support a single overarching message: capacity should be treated as  a first-class constraint of the learning problem, on equal footing with the statistical objective itself.

			\section*{Appendix}
            \appendix
\noindent\textbf{A. Robustness lo large oversampling rates}	\label{appendix_A}		

Figure~5 reports the same results as Fig.~3, but recomputing the SMOTE benchmarks under a more aggressive, non-linear oversampling scheme, with amplification factors reaching up to 20. Panel~A of Fig 4 shows the oversampling factor used in each simulation setting, as a function of $\mathbb{P}(Y=0)$. Panel~B shows the resulting minority-class proportion after applying SMOTE, which is the proportion used to train the algorithm.

For example, when the original sample contains only $0.05\%$ observations from the minority class, the SMOTE-augmented sample has approximately $9.5\%$ minority-class observations: $0.05\%$ real observations and about $9\%$ synthetic observations.

As shown in Figure~5, even under such extreme oversampling, the proposed method yields better performance than approaches based on oversampling without a principled criterion.
\begin{figure}
		\centering
		\includegraphics[width=1\textwidth]
{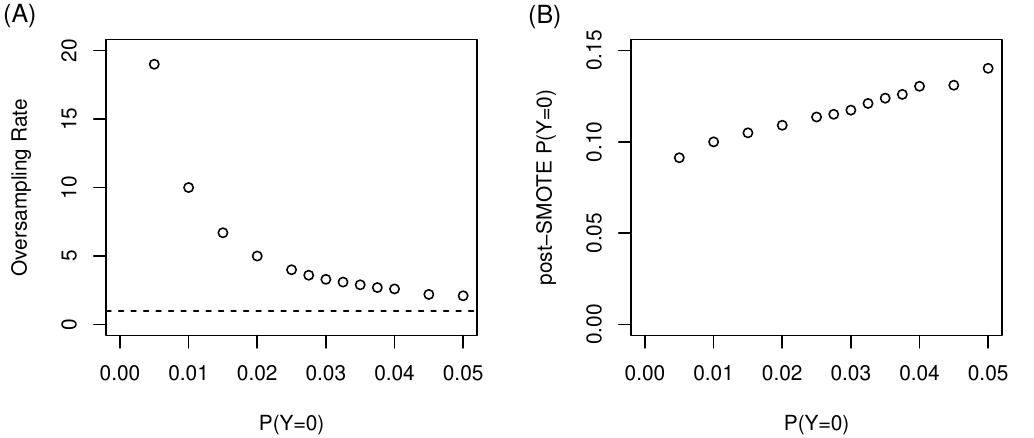}
\caption{
Oversampling scheme as a function of class imbalance. Panel (A) shows the
oversampling factor used for each value of $\pi_0$. Panel (B) shows the resulting minority-class proportion after applying SMOTE, used for training
the algorithm.
}
		\label{pio_post}
	\end{figure}
    
	\begin{figure}
		\centering
		\includegraphics[width=1\textwidth]{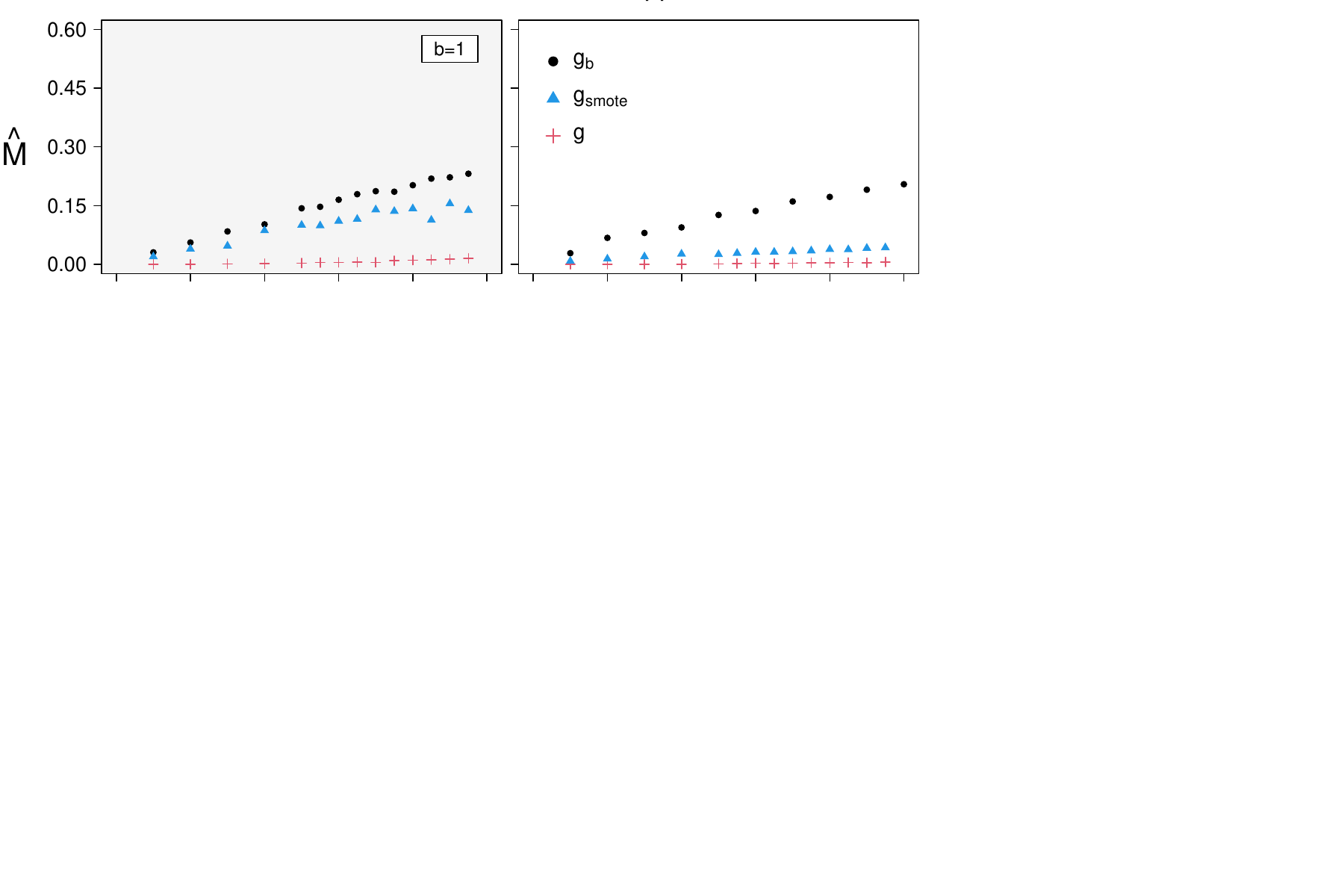}
		\caption{Same experimental setting as in Figure~3, with SMOTE implemented according to
the scheme described in Figure~4.}
		\label{no_lineal}
	\end{figure}

    \bibliographystyle{plainnat}

\bibliography{unbalanced-ref}
\end{document}